\documentclass[11pt]{article}

\usepackage[margin=1in]{geometry}
\usepackage{color}
\usepackage{amssymb}
\usepackage{amsmath}
\usepackage{amsthm}
\usepackage{url}
\usepackage{amsfonts}
\usepackage{hyperref}
\usepackage{authblk}

\usepackage{enumitem}
\setlist{nolistsep}

\usepackage{graphicx}

\newtheorem{theorem}{Theorem}

\newtheorem{lemma}[theorem]{Lemma}
\newtheorem{corollary}[theorem]{Corollary}
\newtheorem{conjecture}[theorem]{Conjecture}

\begin{document}

\title{Clustering subgaussian mixtures by semidefinite programming}

\author[1]{Dustin G.\ Mixon}
\author[2]{Soledad Villar}
\author[2]{Rachel Ward}
\affil[1]{Department of Mathematics and Statistics, Air Force Institute of Technology}
\affil[2]{Department of Mathematics, University of Texas at Austin}

\date{}
\maketitle

\begin{abstract}
We introduce a model-free relax-and-round algorithm for $k$-means clustering based on a semidefinite relaxation due to Peng and Wei \cite{Peng}.
The algorithm interprets the SDP output as a denoised version of the original data and then rounds this output to a hard clustering.  
We provide a generic method for proving performance guarantees for this algorithm, and we analyze the algorithm in the context of subgaussian mixture models.
We also study the fundamental limits of estimating Gaussian centers by $k$-means clustering in order to compare our approximation guarantee to the theoretically optimal $k$-means clustering solution.  

\end{abstract}

\section{Introduction}

Consider the following mixture model:
For each $t\in[k]:=\{1,\ldots,k\}$, let $\mathcal D_t$ be an unknown subgaussian probability distribution over $\mathbb R^m$, with first moment $\gamma_t\in\mathbb{R}^m$ and second moment matrix with largest eigenvalue $\sigma_t^2$.
For each $t$, an unknown number $n_t$ of random points $\{x_{t,i}\}_{i\in[n_t]}$ is drawn independently from $\mathcal D_t$.
Given the points $\{x_{t,i}\}_{i\in[n_t],t\in[k]}$ along with the model order $k$, the goal is to approximate the centers $\{\gamma_t\}_{t\in[k]}$.
How large must $
\Delta:=\min_{a\neq b}\|\gamma_a-\gamma_b\|_2$ be relative to $\sigma_{\max}:=\max_t \sigma_t$, and how large must $n_\mathrm{min}:=\min_t n_t$ be relative to $n_\mathrm{max}:=\max_t n_t$, in order to have sufficient signal for successful approximation?

For the most popular instance of this problem, where the subgaussian distributions are Gaussians, theoretical guarantees date back to the work of Dasgupta \cite{Dasgupta99}.
Dasgupta introduced an algorithm based on random projections and showed that this algorithm well-approximates centers of Gaussians in $\mathbb R^m$ that are separated by $\sigma_{\max} \sqrt m$.
Since Dasgupta's seminal work, performance guarantees for several algorithmic alternatives have emerged, including expectation maximization \cite{DasguptaS07},
spectral methods \cite{Vempala04, KumarK10, Awasthi12}, projections (random and deterministic) \cite{MV10, Arora01}, and the method of moments \cite{MV10}.
Every existing performance guarantee has one two forms:
\begin{itemize}
\item[(a)]
the algorithm correctly clusters all points according to Gaussian mixture component, or
\item[(b)]
the algorithm well-approximates the center of each Gaussian (a la Dasgupta \cite{Dasgupta99}).
\end{itemize}

Results of type (a), which include \cite{Vempala04, KumarK10, Awasthi12, AchlioptasM05}, require the minimum separation between the Gaussians centers to have a multiplicative factor of $\operatorname{polylog} N$, where $N=\sum_{t=1}^kn_t$ is the total number of points.
This stems from a requirement that every point be closer to their Gaussian's center (in some sense) than the other centers, so that the problem of cluster recovery is well-posed.
We note that in the case of spherical Gaussians, such highly separated Gaussian components may be truncated so as to match a different data model known as the stochastic ball model, where the semidefinite program we use in this paper is already known to be tight with high probability \cite{awasthi15, igmixon15}. 

Results of type (b) tend to be specifically tailored to exploit unique properties of the Gaussian distribution, and are thus not easily generalizable to other data models.
For instance, if $x$ has distribution $\mathcal N(\mu, \sigma^2 I_m)$, then $\mathbb E(\|x-\mu \|^2) = m \sigma^2$, and concentration of measure implies that in high dimensions, most of the points will reside in a thin shell with center $\mu$ and radius about $\sqrt m \sigma$.
This sort of behavior can be exploited to cluster even concentric Gaussians as long as the covariances are sufficiently different.
However, algorithms that perform well even with no separation between centers require a sample complexity which is exponential in $k$ \cite{MV10}. 

In this paper, we provide a performance guarantee of type (b), but our approach is model-free.
In particular, we consider the $k$-means clustering objective:
\begin{alignat}{2}
\label{kmeans}
&\text{minimize}  &       &
\sum_{t=1}^k \sum_{i\in A_t} \bigg\| x_i-\frac{1}{|A_t|}\sum_{j\in A_t} x_j \bigg\|_2^2\\
\nonumber
& \text{subject to}& \quad &
\begin{aligned}[t]
A_1\cup\cdots\cup A_k&=\{1,\ldots,N\},\quad A_i\cap A_j=\emptyset\quad\forall i,j\in[k],~i\neq j
\end{aligned}
\end{alignat}
Letting $D$ denote the $N\times N$ matrix defined entrywise by $D_{ij}=\|x_i-x_j\|_2^2$, then a straightforward calculation gives the following ``lifted'' expression for the $k$-means objective: 
\begin{equation}
\label{eq.lifted kmeans}
\sum_{t=1}^k \sum_{i\in A_t} \bigg\| x_i-\frac{1}{|A_t|}\sum_{j\in A_t} x_j \bigg\|_2^2 = \frac{1}{2} \operatorname{Tr}(DX),\qquad X_{ij}=\left\{\begin{matrix} \frac{1}{|A_t|} & \text{if } i,j\in A_t \\ 
 0 & \text{otherwise} \end{matrix}\right.
\end{equation}
The matrix $X$ necessarily satisfies various convex constraints, and relaxing to such constraints leads to the following semidefinite relaxation of \eqref{kmeans}, first introduced by Peng and Wei in~\cite{Peng}:
\begin{alignat}{2}
\label{eq.kmeansSDP}
& \text{minimize}  &       & \operatorname{Tr}(DX) \\
\nonumber
& \text{subject to}& \quad & 
\begin{aligned}[t]
\operatorname{Tr}(X)&=k,~
X1=1,~
X\geq 0,~
X\succeq0
\end{aligned}
\end{alignat}
Here, $X \geq 0$ means that $X$ is entrywise nonnegative, whereas $X\succeq0$ means that $X$ is symmetric and positive semidefinite.

As mentioned earlier, this semidefinite relaxation is known to be tight for a particular data model called the stochastic ball model~\cite{nellore13, awasthi15, igmixon15}.
In this paper, we study its performance under subgaussian mixture models, which include the stochastic ball model and the Gaussian mixture model as special cases.
The SDP is not typically tight under this general model, but the optimizer can be interpreted as a denoised version of the data and can be rounded in order to produce a good estimate for the centers (and therefore produce a good clustering).

To see this, let $P$ denote the $m\times N$ matrix whose columns are the points $\{x_{t,i}\}_{i\in[n_t],t\in[k]}$.
Notice that whenever $X$ has the form \eqref{eq.lifted kmeans}, then for each $t\in[k]$, $PX$ has $|A_t|$ columns equal to the centroid of points assigned to $A_t$.
In particular, if $X$ is $k$-means-optimal, then $PX$ reports the $k$-means-optimal centroids (with appropriate multiplicities).
Next, we note that every SDP-feasible matrix $X\geq0$ satisfies $X^\top 1=X1=1$, and so $X^\top$ is a stochastic matrix, meaning each column of $PX$ is still a weighted average of columns from $P$.
Intuitively, if the SDP relaxation \eqref{eq.kmeansSDP} were close to being tight, then the SDP-optimal $X$ would make the columns of $PX$ close to the $k$-means-optimal centroids.
Empirically, this appears to be the case (see Figure~\ref{fig.relaxandround} for an illustration).
Overall, we may interpret $PX$ as a denoised version of the original data $P$, and we leverage this strengthened signal to identify good estimates for the $k$-means-optimal centroids.

\begin{figure}
\begin{center}
\includegraphics[width=0.28\textwidth]{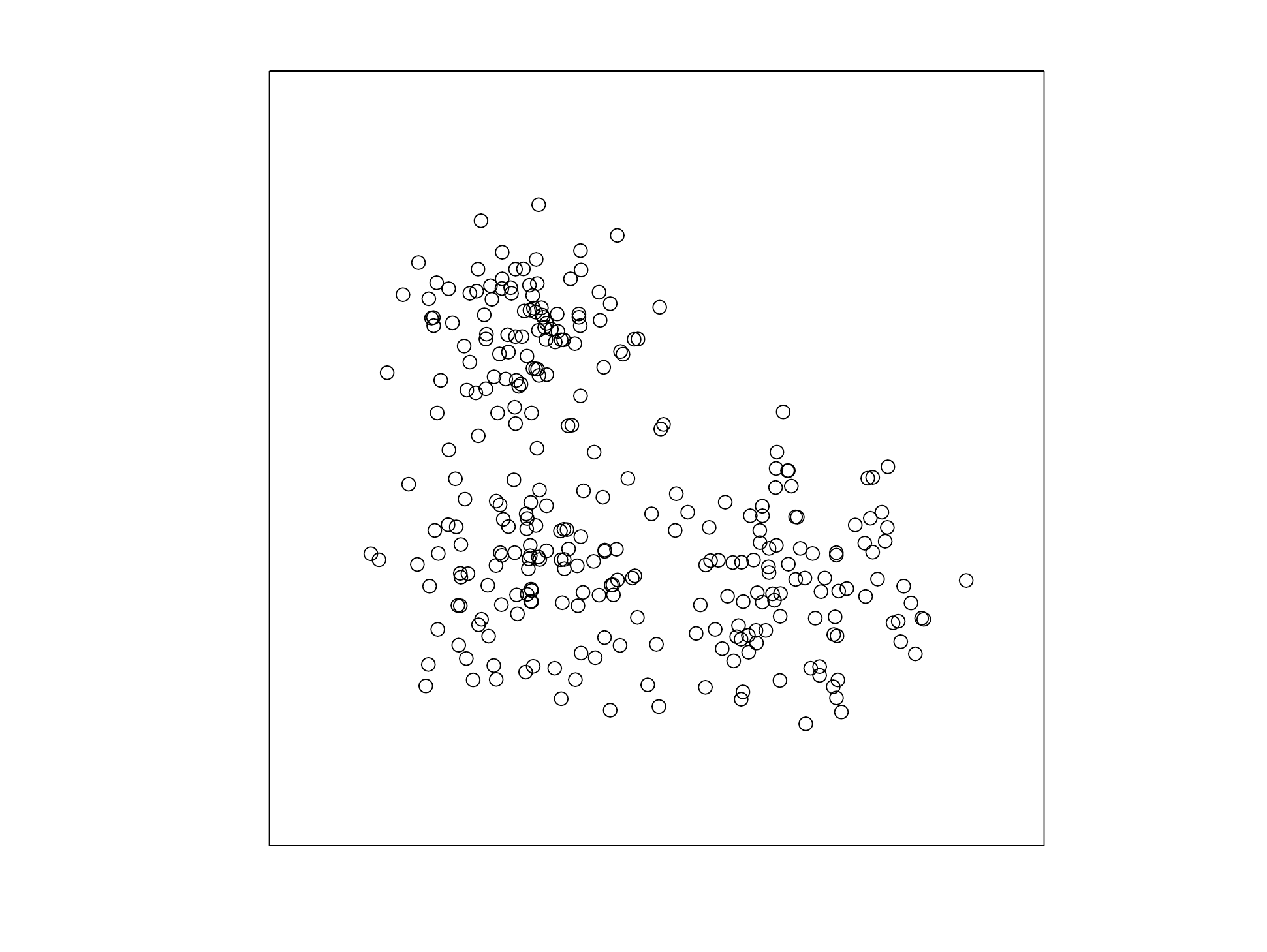}
\hspace{-30pt}
\includegraphics[width=0.28\textwidth]{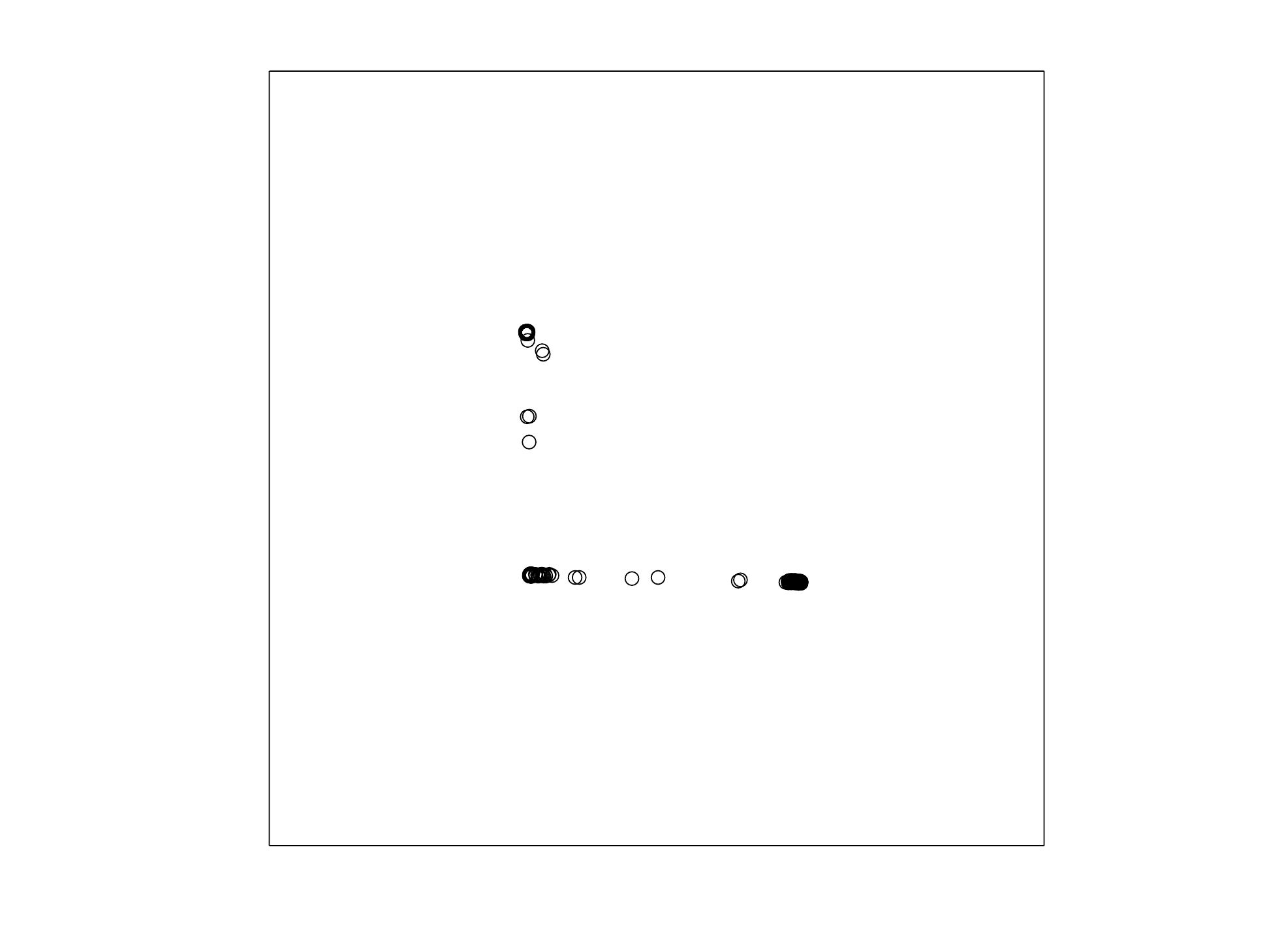}
\hspace{-30pt}
\includegraphics[width=0.28\textwidth]{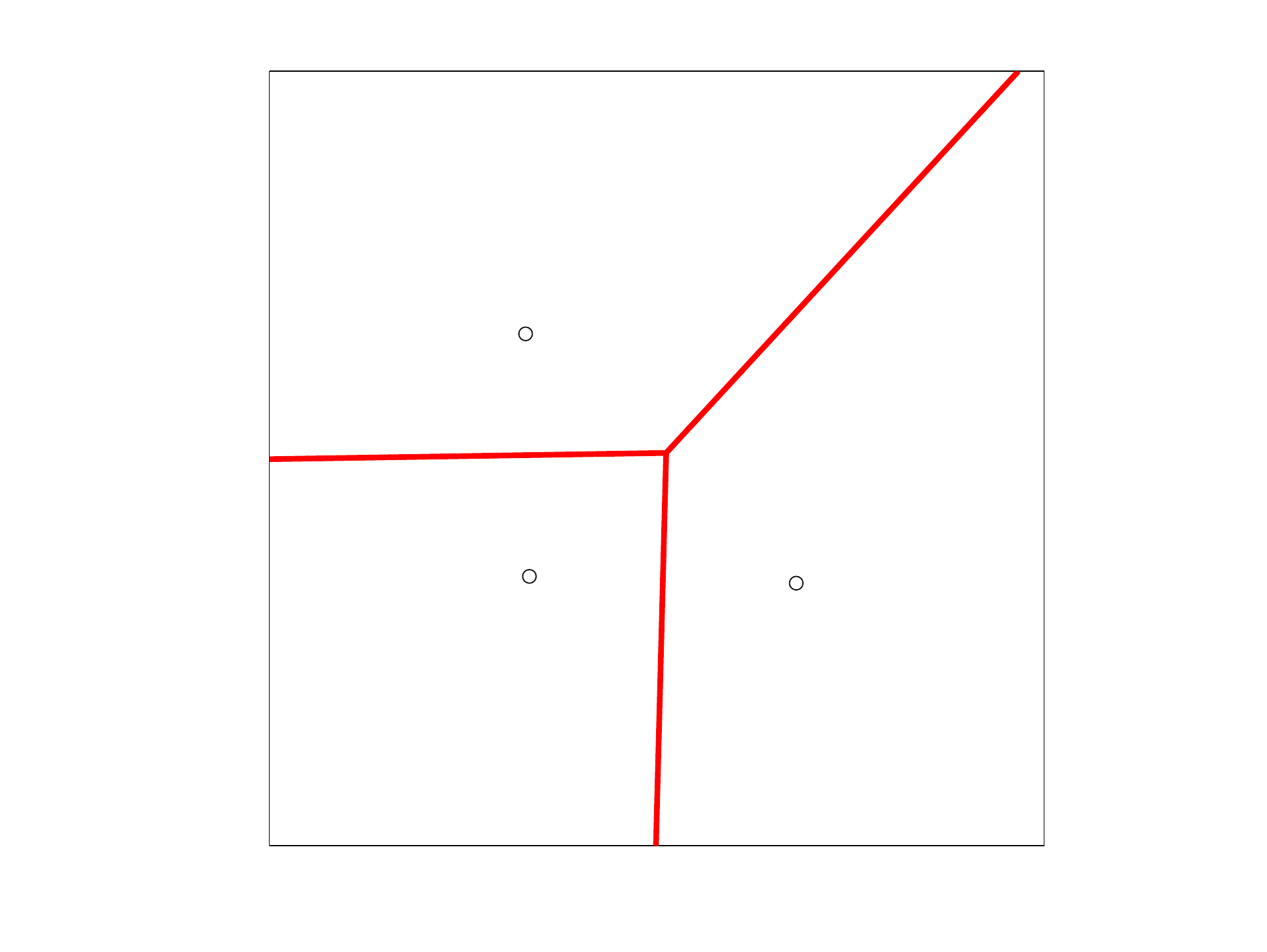}
\hspace{-30pt}
\includegraphics[width=0.28\textwidth]{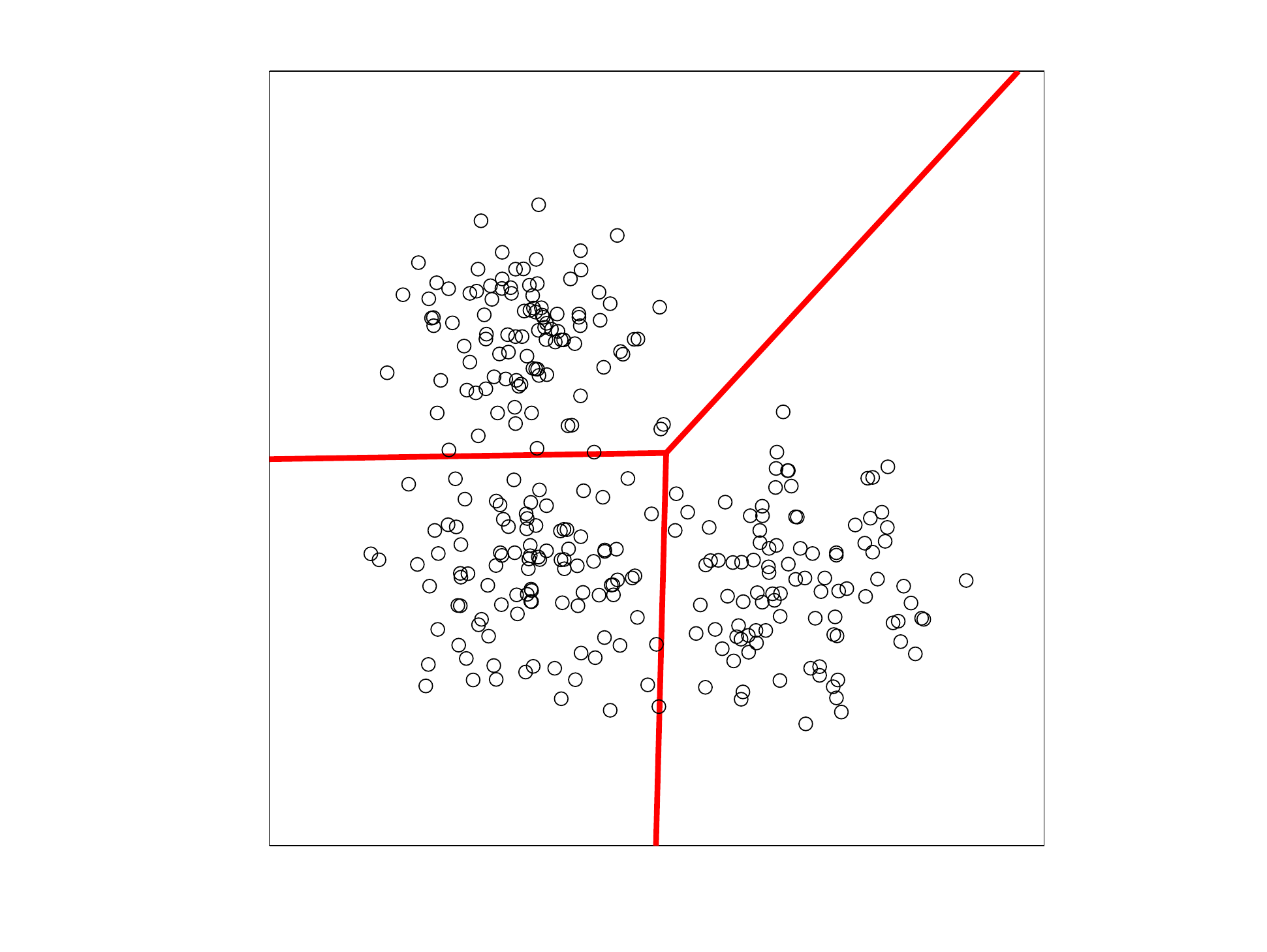}\\
\footnotesize{(a)\hspace{97pt}(b)\hspace{97pt}(c)\hspace{97pt}(d)}
\end{center}
\caption{\footnotesize{\textbf{(a)}
Draw $100$ points at random from each of three spherical Gaussians over $\mathbb{R}^2$.
These points form the columns of a $2\times300$ matrix $P$.
\textbf{(b)}
Compute the $300\times300$ distance-squared matrix $D$ from the data in (a), and solve the $k$-means semidefinite relaxation~\eqref{eq.kmeansSDP} using SDPNAL+v0.3~\cite{yang2015sdpnal+}.
(The computation takes about 16 seconds on a standard MacBook Air laptop.)
Given the optimizer $X$, compute $PX$ and plot the columns.
We interpret this as a denoised version of the original data $P$.
\textbf{(c)}
The points in (b) land in three particular locations with particularly high frequency.
Take these locations to be estimators of the centers of the original Gaussians.
\textbf{(d)}
Use the estimates for the centers in (c) to partition the original data into three subsets, thereby estimating the $k$-means-optimal partition.}
\label{fig.relaxandround}}
\end{figure} 

What follows is a summary of our relax-and-round procedure for (approximately) solving the $k$-means problem~\eqref{kmeans}:

\begin{center}
\fbox{\parbox{0.8\textwidth}{\begin{center}\parbox{0.75\textwidth}{
\textbf{Relax-and-round $k$-means clustering procedure.}\\
Given and $m\times N$ data matrix $P=[x_1\cdots x_N]$, do:
\begin{itemize}
\item[(i)]
Compute distance-squared matrix $D$ defined by $D_{ij}=\|x_i-x_j\|_2^2$.
\item[(ii)]
Solve $k$-means semidefinite program \eqref{eq.kmeansSDP}, resulting in optimizer $X$.
\item[(iii)]
Cluster the columns of the denoised data matrix $PX$.
\end{itemize}
}\end{center}}}
\end{center}

For step (iii), we find there tends to be $k$ vectors that appear as columns in $PX$ with particularly high frequency, and so we are inclined to use these as estimators for the $k$-mean-optimal centroids (see Figure~\ref{fig.relaxandround}, for example).
Running Lloyd's algorithm for step (iii) also works well in practice.
To obtain theoretical guarantees, we instead find the $k$ columns of $PX$ for which the unit balls of a certain radius centered at these points in $\mathbb{R}^m$ contain the most columns of $PX$ (see Theorem~\ref{thm.rounding} for more details).
An implementation of our procedure is available on GitHut~\cite{kmeanssdp2016}.

\textbf{Our contribution.} 
We study performance guarantees for the $k$-means semidefinite relaxation \eqref{eq.kmeansSDP} when the point cloud is drawn from a subgaussian mixture model.
Adapting ideas from Gu\'{e}don and Vershynin \cite{guedon14}, we obtain approximation guarantees comparable with the state of the art for learning mixtures of Gaussians despite the fact that our algorithm is a generic $k$-means solver and uses no model assumptions.
To the best of our knowledge, no convex relaxation has been used before to provide theoretical guarantees for clustering mixtures of Gaussians.
We also provide conditional lower bounds on how well a $k$-means solver can approximate the centers of Gaussians.

\textbf{Organization of this paper.} 
In Section~\ref{sec:summary}, we present a summary of our results and give a high-level explanation of our proof techniques.
We also illustrate the performance of our relax-and-round algorithm on the MNIST dataset of handwritten digits.
Our theoretical results consist of an approximation theorem for the SDP (proved in Section \ref{section.proof of theorem}), a denoising consequence of the approximation (explained in Section~\ref{sec:denoising}), and a rounding step (presented in Section~\ref{sec:rounding}).
We also study the fundamental limits for estimating Gaussian centers by $k$-means clustering (see Section~\ref{sec:lower_bounds}).

\section{Summary of results}\label{sec:summary}

This paper has two main results.
First, we present a relax-and-round algorithm for $k$-means clustering that well-approximates the centers of sufficiently separated subgaussians.
Second, we provide a conditional result on the minimum separation necessary for Gaussian center approximation by $k$-means clustering.
The first result establishes that the $k$-means SDP~\eqref{eq.kmeansSDP} performs well with noisy data (despite not being tight), and the second result helps to illustrate how sharp our analysis is.
This section discusses these results, and then applies our algorithm to the MNIST dataset~\cite{lecun2010mnist}.

\subsection{Performance guarantee for the $k$-means SDP}
Our relax-and-round performance guarantee consists of three steps.

\textbf{Step 1: Approximation.}
We adapt an approach used by Gu\'edon and Vershynin to provide approximation guarantees for a certain semidefinite program under the stochastic block model for graph clustering \cite{guedon14}.

Given the points $x_{t,1},\ldots,x_{t,n_t}$ drawn independently from $\mathcal D_t$, consider the squared-distance matrix $D$ and the corresponding minimizer $X_D$ of the SDP \eqref{eq.kmeansSDP}.
We first construct a ``reference" matrix $R$ such that the SDP \eqref{eq.kmeansSDP} is tight when $D = R$ with optimizer $X_R$. 
To this end, take $\Delta_{ab}:=\|\gamma_a-\gamma_b\|_2$, let $X_D$ denote the minimizer of \eqref{eq.kmeansSDP}, and let $X_R$ denote the minimizer of \eqref{eq.kmeansSDP} when $D$ is replaced by the reference matrix $R$ defined as 
\begin{equation}
\label{eq:reference}
(R_{ab})_{ij} := \xi + \Delta_{ab}^2/2 +  \max\left\{ 0, \Delta_{ab}^2/2 + 2 \langle r_{a,i}-r_{b,j}, \gamma_a-\gamma_b \rangle\right\}
\end{equation}
where $r_{t,i}:=x_{t,i}-\gamma_t$, and $\xi>0$ is a parameter to be chosen later. Indeed, this choice of reference is quite technical, as an artifact of the entries in $D$ being statistically dependent. Despite its lack of beauty, our choice of reference enjoys the following important property:

\begin{lemma}
\label{lem.XR}
Let $1_a\in\mathbb{R}^N$ denote the indicator function for the indices $i$ corresponding to points $x_i$ drawn from the $a$th subgaussian.
If $\gamma_a\neq\gamma_b$ whenever $a\neq b$, then $X_R=\sum_{t=1}^k(1/n_t)1_t1_t^\top$.
\end{lemma}

\begin{proof}
Let $X$ be feasible for the the SDP \eqref{eq.kmeansSDP}.  Replacing $D$ with $R$ in \eqref{eq.kmeansSDP}, we may use the SDP constraints $X1=1$ and $X \geq 0$ to obtain the bound
\begin{align*}
\operatorname{Tr}(RX)
&=\sum_{i=1}^N\sum_{j=1}^N R_{ij} X_{ij} \geq \sum_{i=1}^N\sum_{j=1}^N \xi X_{ij} =  \sum_{i=1}^N\xi \sum_{j=1}^N  X_{ij} =  N \xi =\operatorname{Tr}(RX_R)
\end{align*}
Furthermore, since $\gamma_a\neq\gamma_b$ whenever $a\neq b$, and since $X\geq0$, we have that equality occurs precisely for the $X$ such that $(X_{ab})_{ij}$ equals zero whenever $a\neq b$.
The other constraints on $X$ then force $X_R$ to have the claimed form (i.e., $X_R$ is the unique minimizer).
\end{proof}

Now that we know that $X_R$ is the solution we want, it remains to demonstrate regularity in the sense that $X_D \approx X_R$ provided the subgaussian centers are sufficiently separated. For this, we use the following scheme:
\begin{itemize}
\item If $\langle R, X_D \rangle \approx \langle R, X_R \rangle $ then $\|X_D-X_R\|_F^2$ is small (Lemma \ref{lem.bound on distance}).
\item If $D\approx R$ (in some specific sense) then $\langle R, X_D \rangle \approx \langle R, X_R \rangle $ (Lemmas \ref{lem.trace to delta} and \ref{lem.bounds on delta}).
\item If the centers are separated by $O(k\sigma_{\max})$, then $D\approx R$.
\end{itemize}
What follows is the result of this analysis:

\begin{theorem}
\label{thm.distance typically small}
Fix $\epsilon, \eta > 0$.    There exist universal constants $C, c_1,c_2, c_3$ such that if 
$$\alpha = n_{\max}/n_{\min} \lesssim k \lesssim m \quad \text{and} \quad N > \max \{ c_1 m, c_2 \log(2/\eta), \log(c_3/\eta) \},$$
then
$\|X_D-X_R\|^2_\mathrm{F}\leq \epsilon$ with probability $\geq1-2\eta$ provided $$\Delta_\mathrm{min}^2\geq \frac{C}{\epsilon}k^2 \alpha \sigma_{\max}^2$$
where $\Delta_\mathrm{min} = \min_{a \neq b} \| \gamma_a - \gamma_b \|_2$ is the minimal cluster center separation.
\end{theorem}

See Section~\ref{section.proof of theorem} for the proof. Note that if we remove the assumption $\alpha \lesssim k \lesssim m$, we obtain the result $\Delta_\mathrm{min}^2\geq \frac{C}{\epsilon}(\min\{k,m\} + \alpha) k \alpha \sigma_{\max}^2$.

\textbf{Step 2: Denoising.} Suppose we solve the SDP \eqref{eq.kmeansSDP} for an instance of the subgaussian mixture model where $\Delta_\mathrm{min}$ is sufficiently large.
Then Theorem~\ref{thm.distance typically small} ensures that the solution $X_D$ is close to the ground truth.
At this point, it remains to convert $X_D$ into an estimate for the centers $\{\gamma_t\}_{t\in[k]}$.
Let $P$ denote the $m\times N$ matrix whose $(a,i)$th column is $x_{a,i}$.
Then $PX_R$ is an $m\times N$ matrix whose $(a,i)$th column is $\tilde \gamma_a$, the centroid of the $a$th cluster, which converges to $\gamma_a$ as $N\rightarrow\infty$, (and does not change when $i$ varies, for a fixed $a$), and so one might expect $PX_D$ to have its columns be close to the $\gamma_t$'s.
In fact, we can interpret the columns of  $PX_D$ as a denoised version of the points (see Figure~\ref{fig.relaxandround}).

To illustrate this idea, assume the points $\{x_{a,i}\}_{i\in[n]}$ come from $\mathcal N(\gamma_a, \sigma^2 I_m)$ in $\mathbb R^m$ for each $a\in[k]$. 
Then we have \begin{equation}
\mathbb{E}\bigg[\frac{1}{N}\sum_{a=1}^k\sum_{i=1}^n\|x_{a,i}-\gamma_a\|_2^2\bigg]
=m\sigma^2.
\end{equation}

Letting $c_{a,i}$ denote the $(a,i)$th column of $PX_D$ (i.e., the $i$th estimate of $\gamma_a$), in Section \ref{sec:denoising} we obtain the following denoising result:

\begin{corollary} \label{cor:denoising}
If $k\sigma\lesssim\Delta_\mathrm{min}\leq\Delta_\mathrm{max}\lesssim K\sigma$, then
\[
\displaystyle{\frac{1}{N}\sum_{a=1}^k\sum_{i=1}^n\|c_{a,i}-\tilde\gamma_a\|_2^2\lesssim K^2\sigma^2}
\]
with high probability as $n\rightarrow\infty$.
\end{corollary}

Note that Corollary \ref{cor:denoising} guarantees denoising in the regime $K\ll \sqrt m$.
This is a corollary of a more technical result (Theorem~\ref{thm.denoise}), which guarantees denoising for certain configurations of subgaussians (e.g., when the $\gamma_t$'s are vertices of a regular simplex) in the regime $k\ll m$.

At this point, we comment that one might expect this level of denoising from principal component analysis (PCA) when the mixture of subgaussians is sufficiently nice.
To see this, suppose we have spherical Gaussians of equal entrywise variance $\sigma^2$ centered at vertices of a regular simplex.
Then in the large-sample limit, we expect PCA to approach the $(k-1)$-dimensional affine subspace that contains the $k$ centers.
Projecting onto this affine subspace will not change the variance of any Gaussian in any of the principal components, and so one expects the mean squared deviation of the projected points from their respective Gaussian centers to be $(k-1)\sigma^2$.

By contrast, we find that in practice, the SDP denoises substantially more than PCA does.
For example, Figures~\ref{fig.relaxandround} and~\ref{figure.mnist} illustrate cases in which PCA would not change the data, since the data already lies in $(k-1)$-dimensional space, and yet the SDP considerably enhances the signal.
In fact, we observe empirically that the matrix $X_D$ has low rank and that $PX_D$ has repeated columns.  
This doesn't come as a complete surprise, considering SDP optimizers are known to exhibit low rank~\cite{shapiro1982rank,barvinok1995problems,pataki1998rank}.
Still, we observe that the optimizer tends to have rank $O(k)$ when clustering points from the mixture model.
This is not predicted by existing bounds, and we did not leverage this feature in our analysis, though it certainly warrants further investigation.

\textbf{Step 3: Rounding.}
In Section \ref{sec:rounding}, we present a rounding scheme that provides a clustering of the original data from the denoised results of the SDP (Theorem \ref{thm.rounding}).  In general, the cost of rounding is a factor of $k$ in the average squared deviation of our estimates. Under the same hypothesis as Corollary~\ref{cor:denoising}, we have that there exists a permutation $\pi$ on $\{1,\ldots, k\}$ such that
\begin{equation}
\label{eq.rounding bound}
\displaystyle{\frac{1}{k}\sum_{i=1}^k\| v_i -\tilde\gamma_{\pi(i)}\|_2^2\lesssim k K^2\sigma^2},
\end{equation}
where $v_i$ is what our algorithm chooses as the $i$th center estimate. 
Much like the denoising portion, we also have a more technical result that allows one to replace the right-hand side of \eqref{eq.rounding bound} with $k^2\sigma^2$ for sufficiently nice configurations of subgaussians.
As such, we can estimate Gaussian centers with mean squared error $O(k^2\sigma^2)$ provided the centers have pairwise distance $\Omega(k\sigma)$.
This contrasts with the seminal work of Dasgupta~\cite{Dasgupta99}, which gives mean squared error $O(m\sigma^2)$ when the pairwise distances are $\Omega(\sqrt{m}\sigma)$.
As such, our guarantee replaces dimensionality dependence with model order--dependence, which is an improvement to the state of the art in the regime $k\ll \sqrt{m}$.
In the next section, we indicate that model order--dependence cannot be completely removed when using $k$-means to estimate the centers.

Before concluding this section, we want to clarify the nature of our approximation guarantee \eqref{eq.rounding bound}.
Since centroids correspond to a partition of Euclidean space, our guarantee says something about how ``close'' our $k$-means partition is to the ``true'' partition.
By contrast, the usual approximation guarantees for relax-and-round algorithms compare values of the objective function (e.g., the $k$-means value of the algorithm's output is within a factor of 2 of minimum).
Also, the latter sort of optimal value--based approximation guarantee cannot be used to produce the sort of optimizer-based guarantee we want.
To illustrate this, imagine a relax-and-round algorithm for $k$-means that produces a near-optimal partition with $k=2$ for data coming from a single spherical Gaussian.
We expect every subspace of co-dimension 1 to separate the data into a near-optimal partition, but the partitions are very different from each other when the dimension $m\geq2$, and so a guarantee of the form \eqref{eq.rounding bound} will not hold.

\subsection{Fundamental limits of $k$-means clustering} \label{sec:lower_bounds}

In Section \ref{sec:rounding}, we provide a rounding scheme that, when applied to the output of the $k$-means SDP, produces estimates of the subgaussian centers.
But how good is our guarantee? 
Observe the following two issues:
(i) The amount of tolerable noise $\sigma$ and our bound on the error $\max_i\|v_i-\tilde\gamma_{\pi(i)}\|_2$ both depend on $k$.
(ii) Our bound on the error does not vanish with $N$.

In this section, we give a conditional result that these issues are actually artifacts of $k$-means; that is, both of these would arise if one were to estimate the Gaussian centers with the $k$-means-optimal centroids (though these centroids might be computationally difficult to obtain).
The main trick in our argument is that, in some cases, so-called ``Voronoi means'' appear to serve as a good a proxy for the $k$-means-optimal centroids.
This trick is useful because the Voronoi means are much easier to analyze.
We start by providing some motivation for the Voronoi means.

Given $\mathcal{X}=\{x_i\}_{i=1}^N\subseteq\mathbb{R}^m$, let $A^{(\mathcal{X})}_1\sqcup\cdots\sqcup A^{(\mathcal{X})}_k=\{1,\ldots,N\}$ denote any minimizer of the $k$-means objective
\[
\sum_{t=1}^k\sum_{i\in A_t}\bigg\|x_i-\frac{1}{|A_t|}\sum_{j\in A_t}x_j\bigg\|_2^2,
\]
and define the \textbf{$k$-means-optimal centroids} by
\[
c^{(\mathcal{X})}_t:=\frac{1}{|A^{(\mathcal{X})}_t|}\sum_{j\in A^{(\mathcal{X})}_t}x_j.
\]
(Note that the $k$-means minimizer is unique for generic $\mathcal{X}$.)
Given $\Gamma=\{\gamma_t\}_{t=1}^k$, then for each $\gamma_a$, consider the Voronoi cell
\[
V_a^{(\Gamma)}:=\Big\{x\in\mathbb{R}^m:\|x-\gamma_a\|_2<\|x-\gamma_b\|_2~\forall b\neq a\Big\}.
\]
Given a probability distribution $\mathcal{D}$ over $\mathbb{R}^m$, define the \textbf{Voronoi means} by
\[
\mu_t^{(\Gamma,\mathcal{D})}:=\mathop{\mathbb{E}}_{X\sim\mathcal{D}}\big[X\big|X\in V_t^{(\Gamma)}\big].
\]
Finally, we say $\Gamma\subseteq\mathbb{R}^m$ is a \textbf{stable isogon} if
\begin{itemize}
\item[(si1)]
$|\Gamma|>1$,
\item[(si2)]
the symmetry group $G\leq O(m)$ of $\Gamma$ acts transitively on $\Gamma$, and
\item[(si3)]
for each $\gamma\in\Gamma$, the stabilizer $G_\gamma$ has the property that
\[
\{x\in\mathbb{R}^m: Qx=x ~~\forall Q\in G_\gamma\}
=\operatorname{span}\{\gamma\}.
\]
\end{itemize}
(Examples of stable isogons include regular and quasi-regular polyhedra, as well as highly symmetric frames~\cite{BroomeW:13}.)
With this background, we formulate the following conjecture:

\begin{conjecture}[Voronoi Means Conjecture]
\label{conj.voronoi}
Draw $N$ points $\mathcal{X}$ independently from a mixture $\mathcal{D}$ of equally weighted spherical Gaussians of equal variance centered at the points in a stable isogon $\Gamma=\{\gamma_t\}_{t=1}^k$.
Then 
\[
\min_{\substack{\pi\colon[k]\rightarrow[k]\\\mathrm{permutation}}}\max_{t\in\{1,\ldots,k\}} \big\|c^{(\mathcal{X})}_t-\mu_{\pi(t)}^{(\Gamma,\mathcal{D})}\big\|_2
\]
converges to zero in probability as $N\rightarrow\infty$, i.e., the $k$-means-optimal centroids converge to the Voronoi means.
\end{conjecture}

Our conditional result will only require the Voronoi Means Conjecture in the special case where the $\gamma_t$'s form an orthoplex (see Lemma~\ref{lem.orthoplex}).
Figure~\ref{figure.1} provides some numerical evidence in favor of this conjecture (specifically in the orthoplex case).
We note that the hypothesis that $\Gamma$ be a stable isogon appears somewhat necessary.
For example, simulations suggest that the conjecture does not hold when $\Gamma$ is three equally spaced points on a line (in this case, the $k$-means-optimal centroids appear to be more separated than the Voronoi means).
The following theorem provides some insight as to why the stable isogon hypothesis is reasonable:

\begin{figure}[t]
\begin{center}
\includegraphics[width=0.3\textwidth]{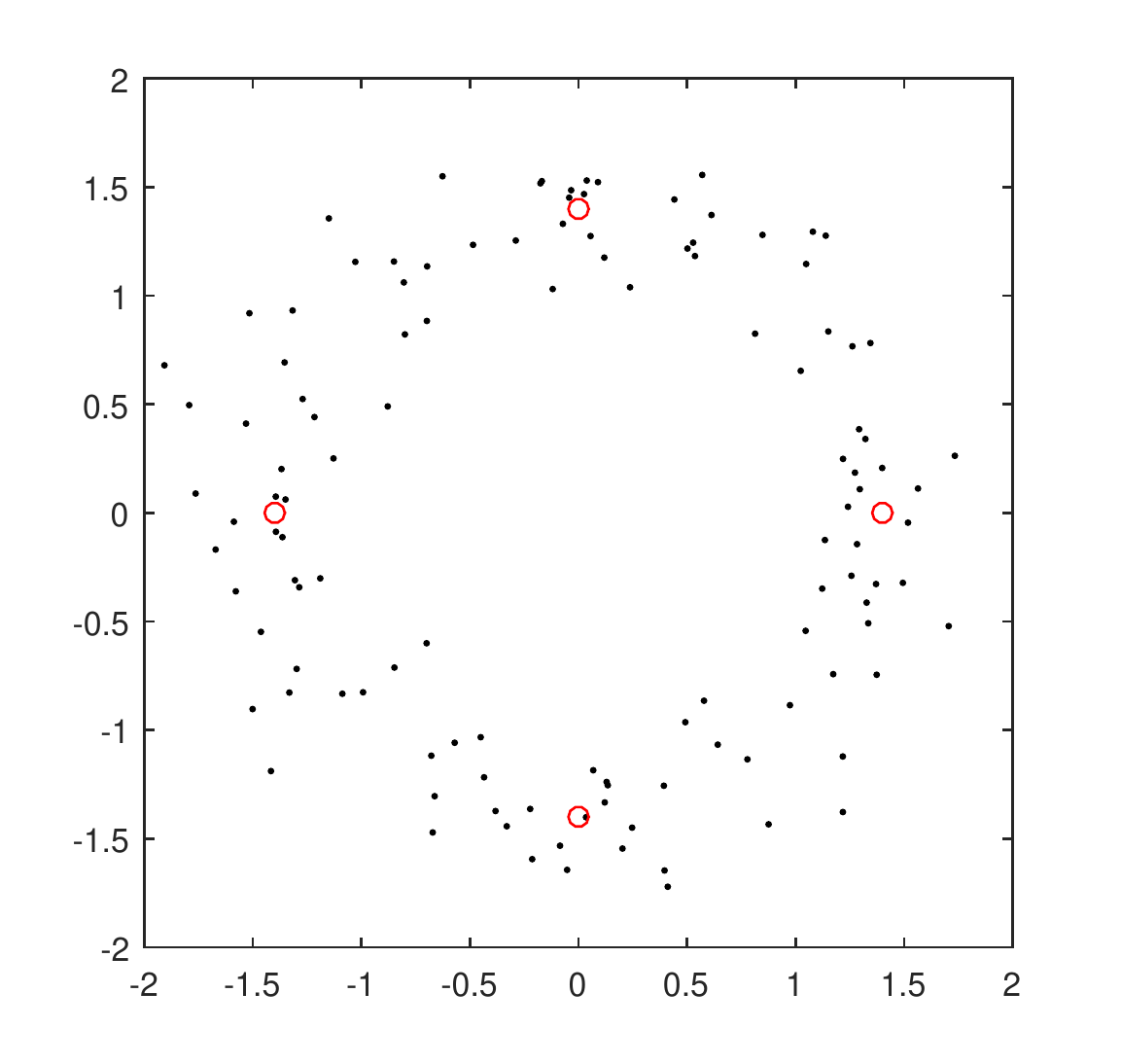}
\includegraphics[width=0.3\textwidth]{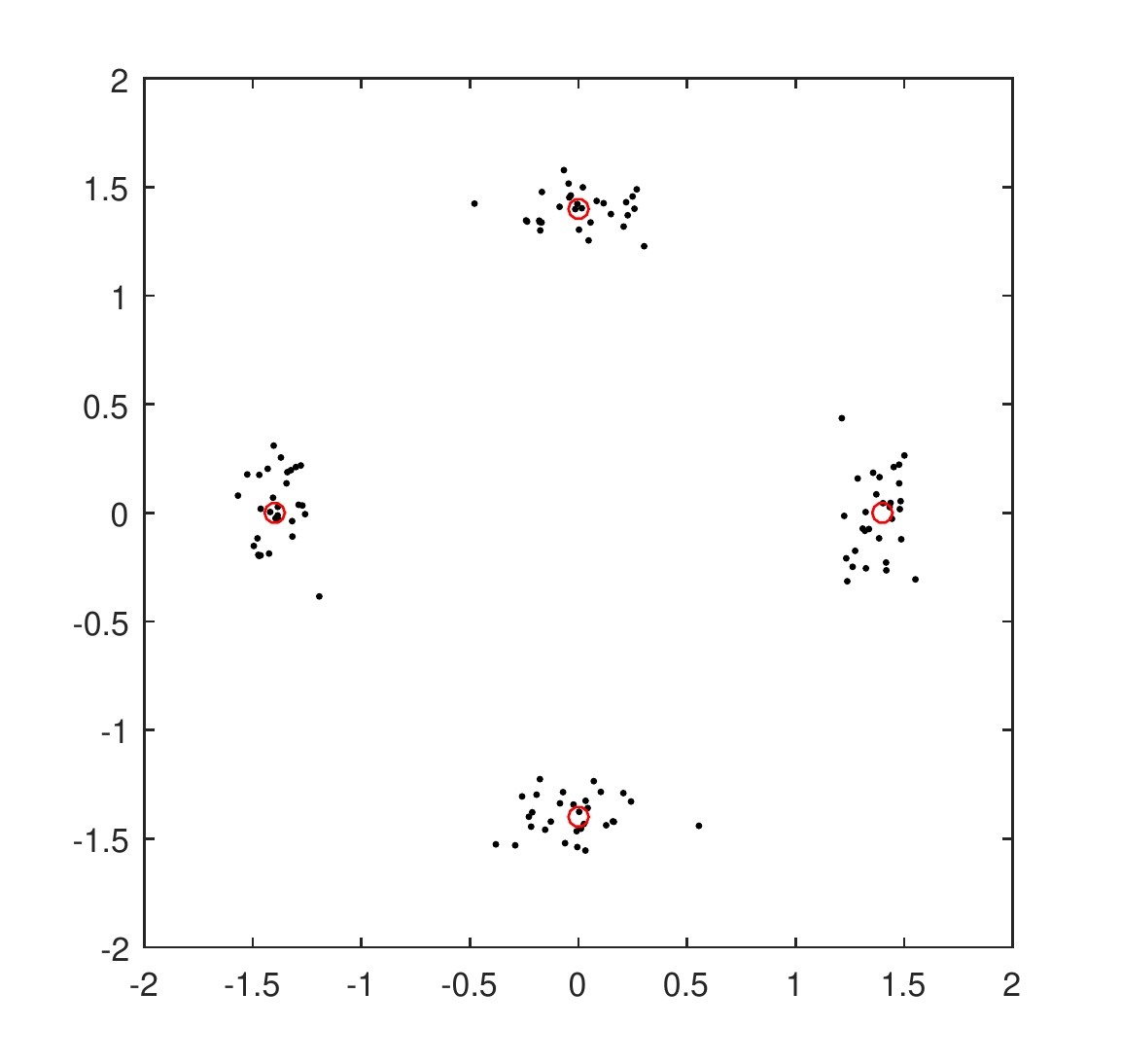}
\includegraphics[width=0.3\textwidth]{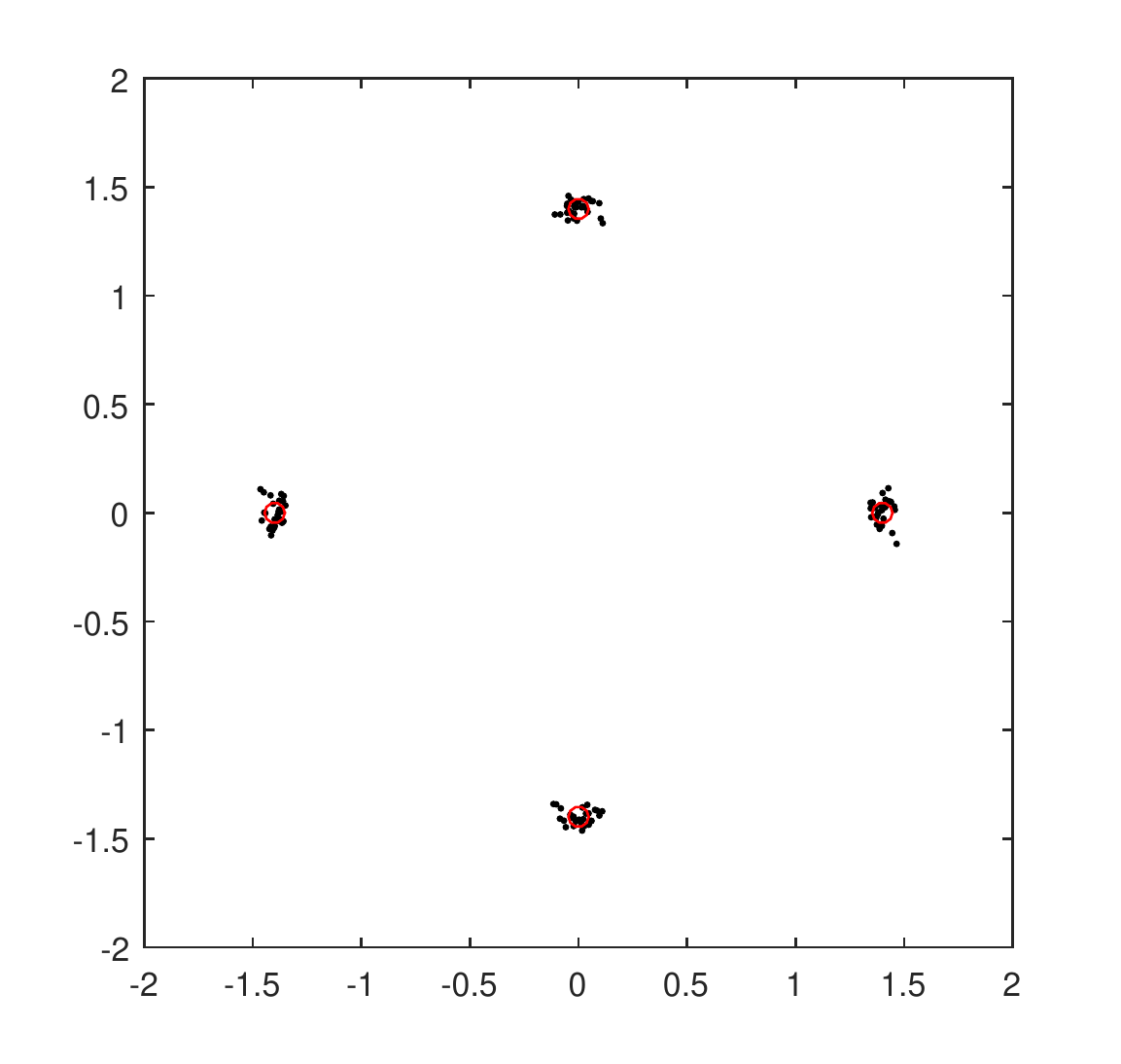}\\
\footnotesize{
\hspace{6pt}(a)\hspace{132pt}(b)\hspace{132pt}(c)}
\end{center}
\caption{
\label{figure.1}
{\footnotesize 
Evidence in favor of the Voronoi Means Conjecture.
For each $n\in\{10^2,10^3,10^4\}$, the following experiment produces (a), (b) and (c), respectively.  Perform $30$ trials of the following:
For each $\gamma\in\{(\pm1,0),(0,\pm1)\}$, draw $n$ points independently from $\mathcal{N}(\gamma,I)$, and then for these $N=4n$ points, run MATLAB's built-in implementation of $k$-means++ with $k=4$ for $10$ independent initializations; of the resulting $10$ choices of centroids, store the ones of minimal $k$-means value (this serves as our proxy for the $k$-means-optimal centroids).
Plot these 30 collections of centroids in black dots, along with the Voronoi means in red circles.
The Voronoi means $\{(\pm\alpha,0),(0,\pm\alpha)\}$ were computed numerically in Mathematica as $\alpha\approx
1.39928$.
Importantly, the $k$-means-optimal centroids appear to converge toward the Voronoi means, not the Gaussian centers, as $N\rightarrow\infty$.}
\normalsize}
\end{figure}

\begin{theorem}
\label{thm.significance of stable isogons}
Let $\mathcal{D}$ be a mixture of equally weighted spherical Gaussians of equal variance centered at the points of a stable isogon $\Gamma=\{\gamma_t\}_{t=1}^k$.
Then there exists $\alpha>0$ such that $\mu_{t}^{(\Gamma,\mathcal{D})}=\alpha\gamma_t$ for each $t\in\{1,\ldots,k\}$.
\end{theorem}

See Section~\ref{sec.proof of thm.significance of stable isogons} for the proof.
To interpret Theorem~\ref{thm.significance of stable isogons}, consider $k$-means optimization over the distribution $\mathcal{D}$ instead of a large sample $\mathcal{X}$ drawn from $\mathcal{D}$.
This optimization amounts to finding $k$ points $C=\{c_t\}_{t=1}^k$ in $\mathbb{R}^m$ that minimize
\begin{equation}
\label{eq.continuum k means}
\sum_{t=1}^k\mathop{\mathbb{E}}_{X\sim\mathcal{D}}\big[\|X-c_t\|_2^2\big|X\in V_t^{(C)}\big]\mathop{\mathbb{P}}_{X\sim\mathcal{D}}\big(X\in V_t^{(C)}\big)
\end{equation}
Intuitively, the optimal $C$ is a good proxy for the $k$-means-optimal centroids when $N$ is large (and one might make this rigorous using the plug-in principle with the Glivenko--Cantelli Theorem).
What Theorem~\ref{thm.significance of stable isogons} provides is that, when $\Gamma$ is a stable isogon, the Voronoi means have the same Voronoi cells as do $\Gamma$.
As such, if one were to initialize Lloyd's algorithm at the Gaussian centers to solve \eqref{eq.continuum k means}, the algorithm converges to the Voronoi means in one step.
Overall, one should interpret Theorem~\ref{thm.significance of stable isogons} as a statement about how the Voronoi means locally minimize \eqref{eq.continuum k means}, whereas the Voronoi Means Conjecture is a statement about global minimization.

As indicated earlier, we will use the Voronoi Means Conjecture in the special case where $\Gamma$ is an orthoplex:

\begin{lemma}
\label{lem.orthoplex}
The standard orthoplex of dimension $d$, given by the first $d$ columns of $I$ and of $-I$, is a stable isogon.
\end{lemma}

\begin{proof}
First, we have that $|\Gamma|=2d>1$, implying (si1).
Next, every $\gamma'\in\Gamma$ can be reached from any $\gamma\in\Gamma$ with the appropriate signed transposition, which permutes $\Gamma$, and is therefore in the symmetry group $G$; as such, we conclude (si2).
For (si3), pick $\gamma\in \Gamma$ and let $i$ denote its nonzero entry.
Consider the matrix $Q=2e_ie_i^\top-I$, where $e_i$ denotes the $i$th identity basis element.
Then $Q\in G_\gamma$, and the eigenspace of $Q$ with eigenvalue $1$ is $\operatorname{span}\{\gamma\}$, and so we are done.
\end{proof}
What follows is the main result of this subsection:
\begin{theorem}
\label{thm.main lower bound}
Let $k\leq 2m$ be even, and let $\Gamma=\{\gamma_t\}_{t=1}^k\subseteq\mathbb{R}^m$ denote the standard orthoplex of dimension $k/2$.
Then for every $\sigma>0$, either
\[
\sigma
\lesssim\Delta_\mathrm{min}/\sqrt{\log k}
\qquad
\text{or}
\qquad
\min_{t\in\{1,\ldots,k\}}\|\mu_t^{(\Gamma,\mathcal{D})}-\gamma_t\|_2
\gtrsim \sigma\sqrt{\log k},
\]
where $\mathcal{D}$ denotes the mixture of equally weighted spherical Gaussians of entrywise variance $\sigma^2$ centered at the members of $\Gamma$. 
\end{theorem}

See Section~\ref{sec.proof of thm.main lower bound} for the proof.
In words, Theorem~\ref{thm.main lower bound} establishes that one must accept $k$-dependence in either the data's noise or the estimate's error.
It would be interesting to investigate whether other choices of stable isogons lead to stronger $k$-dependence.

\subsection{Numerical example: Clustering the MNIST dataset}

In this section, we apply our clustering algorithm to the NMIST handwritten digits dataset~\cite{lecun2010mnist}.
This dataset consists of 70,000 different $28\times28$ grayscale images, reshaped as $784\times1$ vectors; 55,000 of them are considered training set, 10,000 are test set, and the remaining 5,000 are validation set.

Clustering the raw data gives poor results (due to 4's and 9's being similar, for example), so we first learn meaningful features, and then cluster the data in feature space.
To simplify feature extraction, we used the first example from the TensorFlow tutorial~\cite{tensorflow2015}.
This consists of a one-layer neural network $y(x) = \sigma(Wx+b)$, where $\sigma$ is the softmax function, $W$ is a $784\times10$ matrix to learn, and $b$ is a $10\times1$ vector to learn.
As the tutorial shows, the neural network is trained for 1,000 iterations, each iteration using batches of 100 random points from the training set.

Training the neural network amounts to finding $W$ and $b$ that fit the training set well.
After selecting these parameters, we run the trained neural network on the first 1,000 elements of the test set, obtaining $\{y(x_i)\}_{i=1}^{1000}$, where each $y(x_i)$ is a $10\times1$ vector representing the probabilities of being each digit.
Since $y(x_i)$ is a probability vector, its entries sum to $1$, and so the feature space is actually $9$-dimensional.

For this experiment, we cluster $\{y(x_i)\}_{i=1}^{1000}$ with two different algorithms:\ (i) MATLAB's built-in implementation of $k$-means++, and (ii) our relax-and-round algorithm based on the $k$-means semidefinite relaxation~\eqref{eq.kmeansSDP}.
(The results of the latter alternative are illustrated in Figure~\ref{figure.mnist}.)

\begin{figure}[t]
\begin{center}
\includegraphics[width=0.37\textwidth]{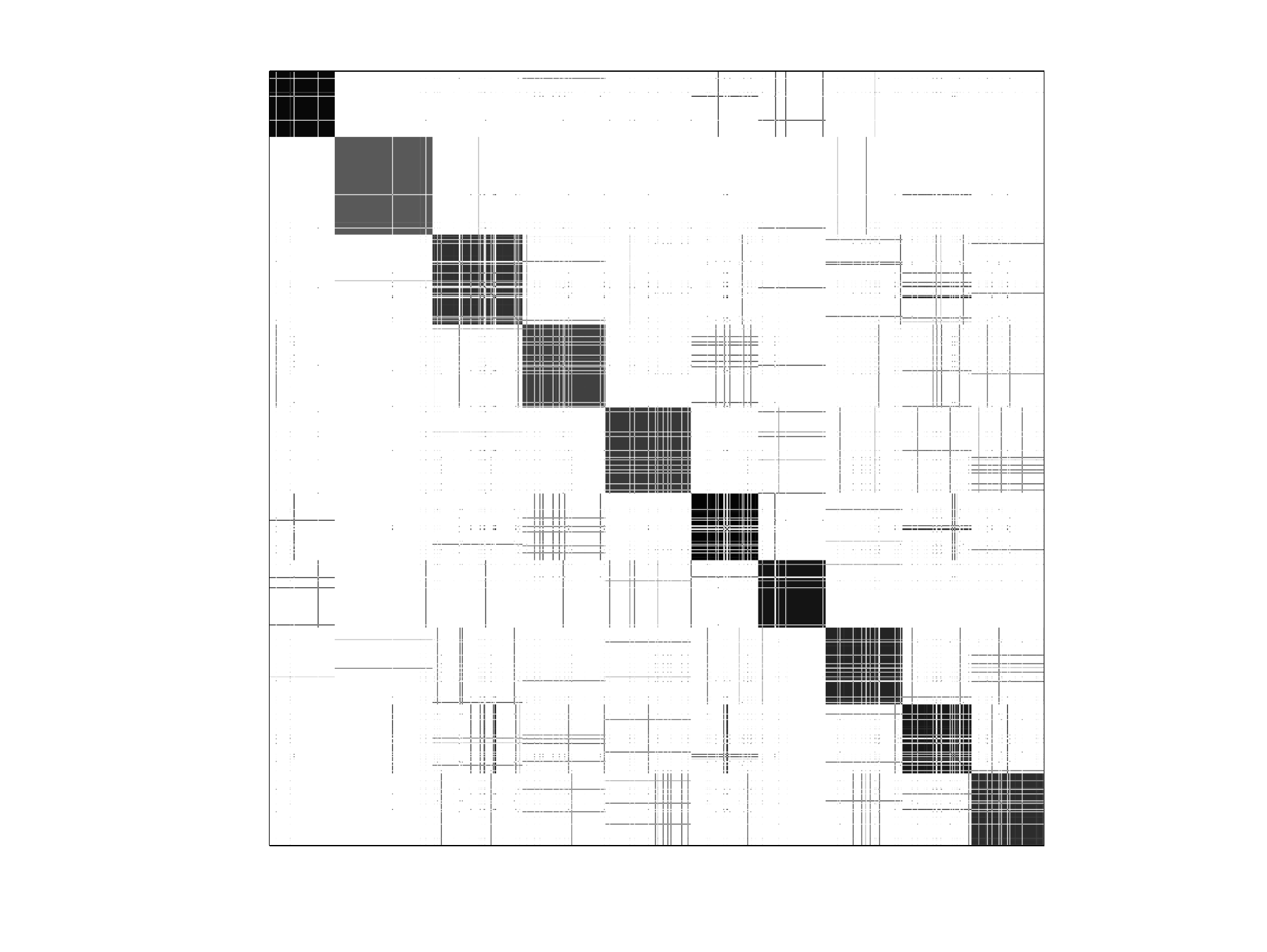}
\hspace{-37pt}
\includegraphics[width=0.37\textwidth]{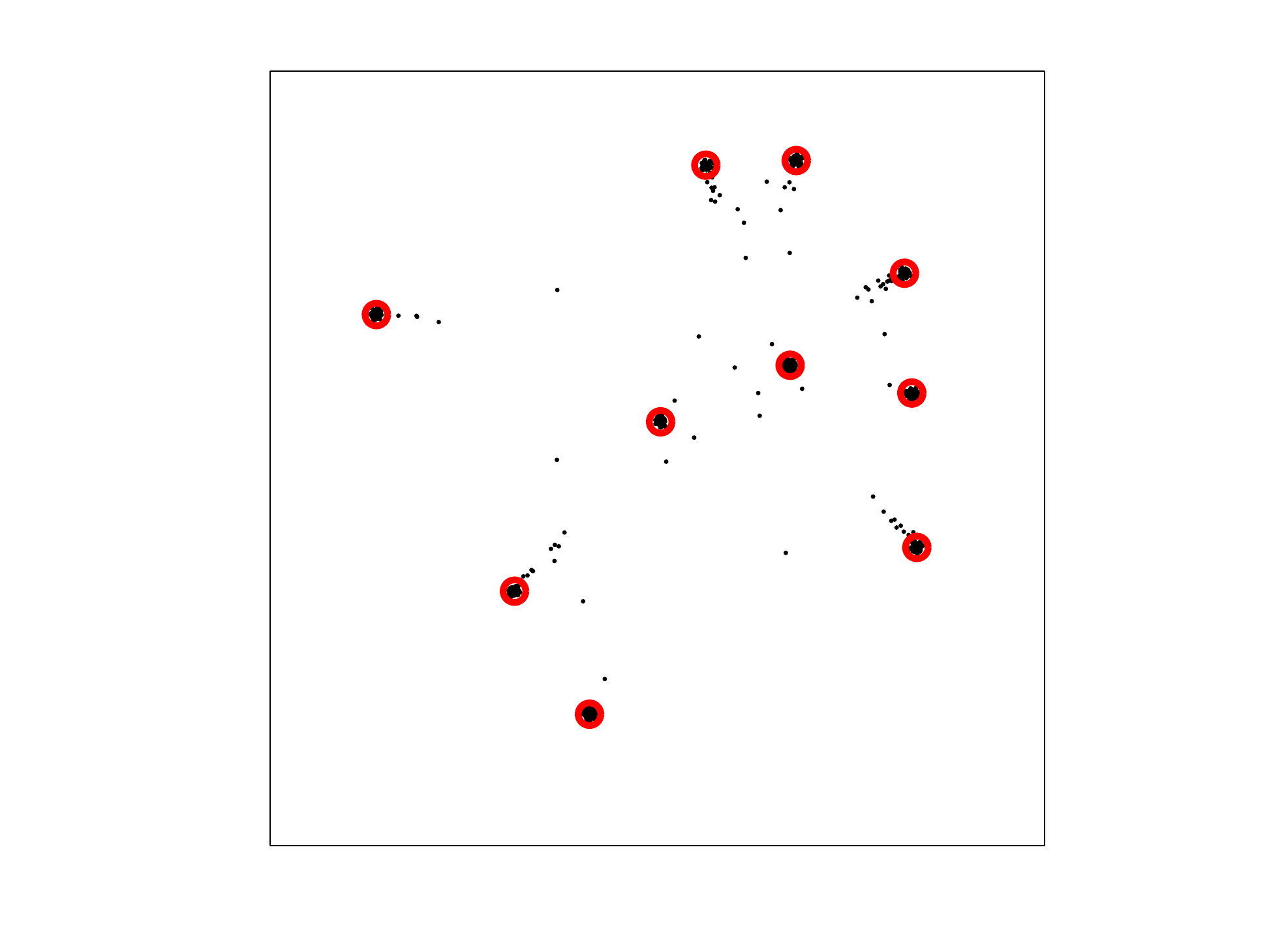}
\hspace{-37pt}
\includegraphics[width=0.37\textwidth]{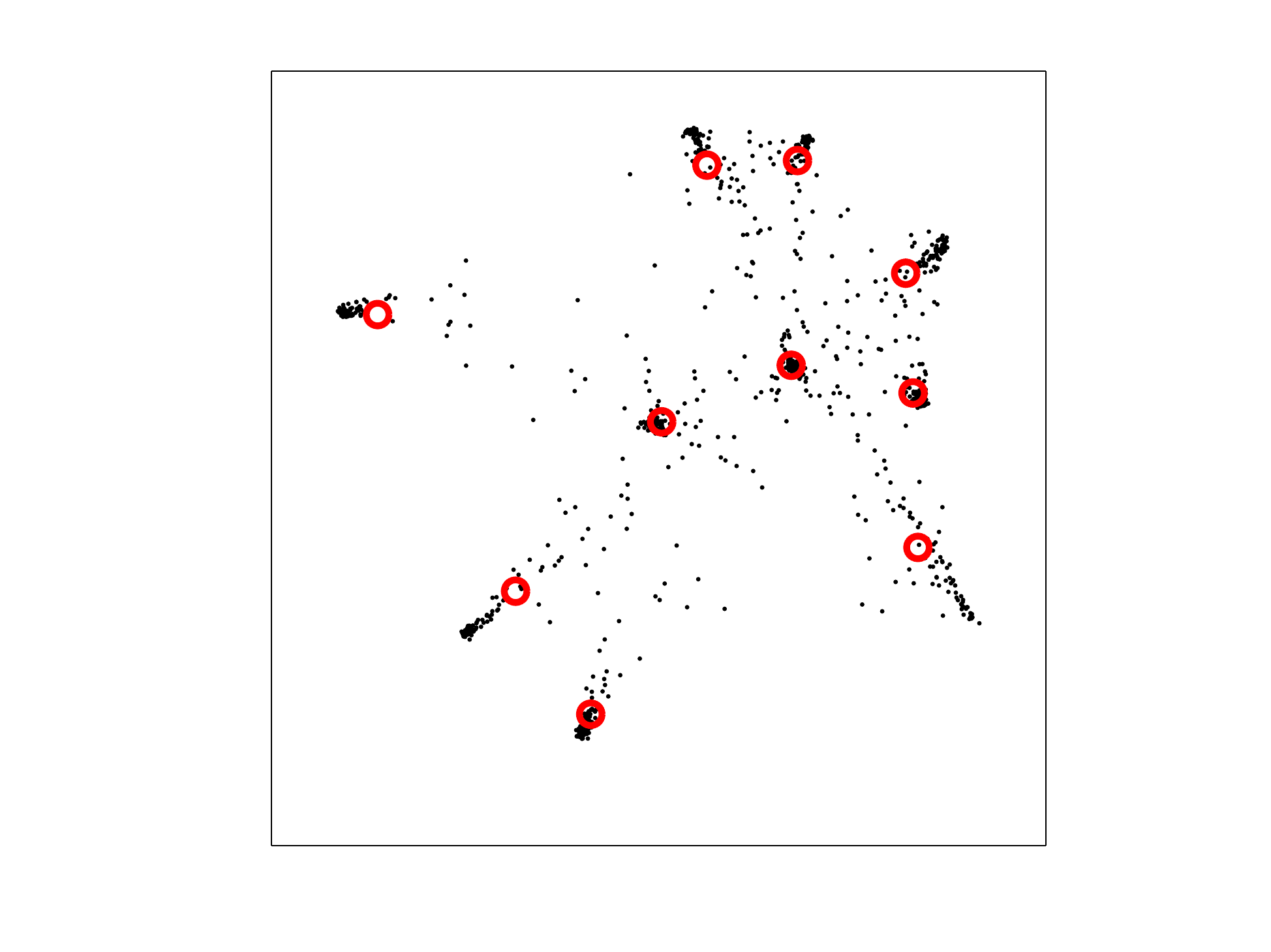}\\
\footnotesize{
\hspace{6pt}(a)\hspace{132pt}(b)\hspace{132pt}(c)}
\end{center}
\caption{\label{figure.mnist}
\footnotesize{
\textbf{(a)}
After applying TensorFlow~\cite{tensorflow2015} to learn a $9$-dimensional feature space of MNIST digits~\cite{lecun2010mnist}, determine the features of the first 1,000 images in the MNIST test set, compute the $1000\times1000$ matrix $D$ of squared distances in feature space, and then solve the $k$-means semidefinite relaxation~\eqref{eq.kmeansSDP} using SDPNAL+v0.3~\cite{yang2015sdpnal+}.
(The computation takes about 6 minutes on a standard MacBook Air laptop.)
Convert the SDP-optimizer $X$ to a grayscale image such that white pixels denote zero entries.
By inspection, this matrix is not exactly of the form~\eqref{eq.lifted kmeans}, but it looks close, and it certainly appears to have low rank.
\textbf{(b)}
Letting $P$ denote the $9\times1000$ matrix whose columns are the feature vectors to cluster, compute the denoised data $PX$ and identify the $10$ most popular locations in $\mathbb{R}^9$ (denoted by red circles) among the columns of $PX$ (denoted by black dots).
For the plot, we project the $9$-dimensional data onto a random $2$-dimensional subspace.
\textbf{(c)}
The $10$ most popular locations form our estimates for the centers of digits in feature space.
We plot these locations relative to the original data, projected in the same $2$-dimensional subspace as (b).
}}
\end{figure}

Since each run of $k$-means++ uses a random initialization that impacts the partition, we ran this algorithm 100 times.
In fact, the $k$-means value of the output varied quite a bit: the all-time low was $39.1371$, the all-time high was $280.4174$, and the median was $108.2358$; the all-time low was reached in $34$ out of the $100$ trials.
Since our relax-and-round alternative has no randomness, the outcome is deterministic, and its $k$-means value was $39.1371$, i.e., identical to the all-time low from $k$-means++.
By comparison, the $k$-means value of the planted solution (i.e., clustering according to the hidden digit label) was 103.5430, and the value of the SDP (which serves as a lower bound on the optimal $k$-means value) was $38.5891$.
As such, not only did our relax-and-round alternative produce the best clustering that $k$-means++ could find, it also provided a certificate that no clustering has a $k$-means value that is 1.5\% better.

Recalling the nature of our approximation guarantees, we also want to know well the relax-and-round algorithm's clustering captures the ground truth.
To evaluate this, we determined a labeling of the clusters for which the resulting classification exhibited a minimal misclassification rate.
(This amounts to minimizing a linear objective over all permutation matrices, which can be relaxed to a generically tight linear program over doubly stochastic matrices.)
For $k$-means++, the all-time low misclassification rate was $0.0971$ (again, accomplished by $34$ of the $100$ trials), the all-time high was $0.4070$, and the median was $0.2083$.
As one might expect, the relax-and-round output had a misclassification rate of $0.0971$.

\section{Proof of Theorem~\ref{thm.distance typically small}}
\label{section.proof of theorem}

By the following lemma, it suffices to bound $\operatorname{Tr}(R(X_D-X_R))$:

\begin{lemma}
\label{lem.bound on distance}
$\displaystyle\|X_D-X_R\|_\mathrm{F}^2\leq \frac{5}{n_{\min}\Delta_\mathrm{min}^2}\operatorname{Tr}(R(X_D-X_R))$.
\end{lemma}

\begin{proof}
First, by Lemma~\ref{lem.XR}, we have $\|X_R\|_\mathrm{F}^2=k$.
We also claim that $\|X_D\|_\mathrm{F}^2\leq k$.
To see this, first note that $X_D1=1$ and $X_D\geq0$, and so the $i$th entry of $X_Dv$ can be interpreted as a convex combination of the entries of $v$.
Let $v$ be an eigenvector of $X_D$ with eigenvalue $\mu$, and let $i$ index the largest entry of $v$ (this entry is positive without loss of generality).
Then $\mu v_i=(X_Dv)_i\leq v_i$, implying that $\mu\leq1$.
Since the eigenvalues of $X_D$ lie in $[0,1]$, we may conclude that $\|X_D\|_\mathrm{F}^2\leq\operatorname{Tr}(X_D)=k$.
As such,
\begin{align}
\|X_D-X_R\|_\mathrm{F}^2
\nonumber
&=\|X_D\|_\mathrm{F}^2+\|X_R\|_\mathrm{F}^2-2\operatorname{Tr}(X_DX_R)\\
\nonumber
&\leq 2k-2\operatorname{Tr}(X_DX_R)\\
\nonumber
&=2k+2\operatorname{Tr}((X_R-X_D)X_R)-2\|X_R\|_\mathrm{F}^2\\
\label{eq.distance1}
&=2\operatorname{Tr}((X_R-X_D)X_R).
\end{align}
We will bound \eqref{eq.distance1} in two different ways, and a convex combination of these bounds will give the result.
For both bounds, we let $\Omega$ denote the indices in the diagonal blocks, and $\Omega^c$ the indices in the off-diagonal blocks, and $\Omega_t \subset \Omega$ denote the indices in the diagonal block for the cluster $t$.
In particular, $A_\Omega$ denotes the matrix that equals $A$ on the diagonal blocks and is zero on the off-diagonal blocks.
For the first bound, we use that $R_\Omega=\xi (11^\top)_\Omega$, and that $(X_R-X_D)_{\Omega} (11^T)_{\Omega}$ has non-negative entries (since both $X_R$ and $X_D$ have non-negative entries, $X_R1=X_D1=1,$ and $X_R=(X_R)_\Omega$).  Recalling that $R_{\Omega} = \xi$, we have
\begin{align}
\nonumber
2\operatorname{Tr}((X_R-X_D)X_R)
&=
\sum_{t=1}^k 2\operatorname{Tr}\left((X_R-X_D)(11^\top)_{\Omega_t} \frac{1}{n_t}\right)
\\ \nonumber
&\geq
\frac{2}{n_{\max}}\operatorname{Tr}\left((X_R-X_D)(11^\top)_{\Omega}\right)\\
&= -\frac{2}{\xi n_{\max}}\operatorname{Tr}\left((X_D-X_R)R_{\Omega}\right)
\label{eq.bound1}\end{align}
For the second bound, we first write $n_{\min} X_R=11^\top-(11^\top)_{\Omega^c} - \sum_{t=1}^k \left(1-\frac{n_{\min}}{n_t}\right)(11^\top)_{\Omega_t}$.
Since $X_R1=1=X_D1$, we then have
\begin{align*}
2\operatorname{Tr}((X_R-X_D)X_R)
&=\frac{2}{n_{\min}}\operatorname{Tr}\left((X_D-X_R)\left((11^\top)_{\Omega^c} +\sum_{t=1}^k \left(1-\frac{n_{\min}}{n_t}\right)(11^\top)_{\Omega_t}  - 11^\top \right)\right)\\
&=\frac{2}{n_{\min}}\operatorname{Tr}\left((X_D-X_R)\left((11^\top)_{\Omega^c} +\sum_{t=1}^k \left(1-\frac{n_{\min}}{n_t}\right)(11^\top)_{\Omega_t}  \right)\right)\\
&\leq \frac{2}{n_{\min}}\operatorname{Tr}((X_D-X_R)(11^\top)_{\Omega^c})\\
&=\frac{2}{n_{\min}}\operatorname{Tr}(X_D(11^\top)_{\Omega^c}),
\end{align*}
where the last and second-to-last steps use that $(X_R)_{\Omega^c}=0$.
Next, $X_D\geq0$ and $R_{\Omega^c}\geq (\xi + \Delta_{\min}^2/2)(11^\top)_{\Omega^c}$, and so we may continue:
\begin{align}
2\operatorname{Tr}((X_R-X_D)X_R)
\nonumber
&\leq\frac{2}{n_{\min}(\xi + \Delta_{\min}^2/2)}\operatorname{Tr}(X_DR_{\Omega^c})\\
\label{eq.bound2}
&=\frac{2}{n_{\min}(\xi + \Delta_{\min}^2/2)}\operatorname{Tr}((X_D-X_R)R_{\Omega^c}),
\end{align}
where again, the last step uses the fact that $(X_R)_{\Omega^c}=0$.
At this point, we have bounds of the form $x\geq ay_1$ with $a<0$ and $x\leq by_2$ with $b>0$ (explicitly, \eqref{eq.bound1} and \eqref{eq.bound2}), and we seek a bound of the form $x\leq c(y_1+y_2)$.
As such, we take the convex combination for $a,b$ such that $a^{-1}/(a^{-1}+b^{-1})<0$ and $b^{-1}/(a^{-1}+b^{-1})>0$
\[
x
\leq \frac{a^{-1}}{a^{-1}+b^{-1}}ay_1+\frac{b^{-1}}{a^{-1}+b^{-1}}by_2
=\frac{1}{a^{-1}+b^{-1}}(y_1+y_2).
\]
Taking $a=-2/(\xi n_{\max})$ and $b=2/(n_{\min}(\xi+\Delta_\mathrm{min}^2/2))$ and combining with \eqref{eq.distance1} then gives
\[
\|X_D-X_R\|_\mathrm{F}^2
\leq2\operatorname{Tr}((X_R-X_D)X_R)
\leq\Big( \frac{\xi}{2}(n_{\min} -n_{\max})+\frac{n_{\min}}{4}\Delta_\mathrm{min}^2\Big)^{-1}\operatorname{Tr}((X_D-X_R)(R_\Omega+R_{\Omega^c})),
\]
choosing $\xi>0$ sufficiently small and simplifying yields the result.
\end{proof}

We will bound $\operatorname{Tr}(R(X_D-X_R))$ in terms of the following:
For each $N\times N$ real symmetric matrix $M$, let ${\cal F}(M)$ denote the value of the following program:
\begin{alignat}{2}
\label{eq.delta}
{\cal F}(M) = ~& \text{maximum}  &     & |\operatorname{Tr}(MX)| \\
\nonumber
& \text{subject to} & \quad & 
\begin{aligned}[t]
\operatorname{Tr}(X)&=k,~
X1=1,~
X\geq0,~
X\succeq0
\end{aligned}
\end{alignat}

\begin{lemma}
\label{lem.trace to delta}
Put $\widetilde{D}:=P_{1^\perp}DP_{1^\perp}$ and $\tilde{R}:=P_{1^\perp}RP_{1^\perp}$.
Then $\operatorname{Tr}(R(X_D-X_R))
\leq 2{\cal F}(\tilde{D}-\tilde{R})$.
\end{lemma}

\begin{proof}
Since $X_D$ and $X_R$ are both feasible in \eqref{eq.delta}, we have
\begin{align*}
-\operatorname{Tr}(\widetilde{D}X_D)+\operatorname{Tr}(\widetilde{R}X_D)
\leq |\operatorname{Tr}((\widetilde{D}-\widetilde{R})X_D)|
&\leq {\cal F}(\widetilde{D}-\widetilde{R}),\\
\operatorname{Tr}(\widetilde{D}X_R)-\operatorname{Tr}(\widetilde{R}X_R)
\leq |\operatorname{Tr}((\widetilde{D}-\widetilde{R})X_R)|
&\leq {\cal F}(\widetilde{D}-\widetilde{R}),
\end{align*}
and adding followed by reverse triangle inequality gives
\begin{equation}
\label{eq.delta1}
2{\cal F}(\widetilde{D}-\widetilde{R})
\geq \Big(\operatorname{Tr}(\widetilde{D}X_R)-\operatorname{Tr}(\widetilde{D}X_D)\Big)+\Big(\operatorname{Tr}(\widetilde{R}X_D)-\operatorname{Tr}(\widetilde{R}X_R)\Big).
\end{equation}
Write $X_{\widetilde{D}}:=P_{1^\perp}X_DP_{1^\perp}$.
Note that $X_D1=(X_D)^{T}1$ implies $X_D=X_{\widetilde{D}}+(1/N)11^\top$, and so
\[
\operatorname{Tr}(\widetilde{D}X_D)
=\operatorname{Tr}(DX_{\widetilde{D}})
=\operatorname{Tr}\Big(D\big(X_D-(1/N)11^\top\big)\Big)
=\operatorname{Tr}(DX_D)+\frac{1}{N}1^\top D1.
\]
Similarly, $\operatorname{Tr}(\widetilde{D}X_R)=\operatorname{Tr}(DX_R)+\frac{1}{N}1^\top D1$, and so
\[
\operatorname{Tr}(\widetilde{D}X_R)-\operatorname{Tr}(\widetilde{D}X_D)
=\operatorname{Tr}(DX_R)-\operatorname{Tr}(DX_D)
\geq0,
\]
where the last step follows from the optimality of $X_D$. 
Similarly, $\operatorname{Tr}(\widetilde{R}X_D)-\operatorname{Tr}(\widetilde{R}X_R)=\operatorname{Tr}(R(X_D-X_R))$, and so \eqref{eq.delta1} implies the result.
\end{proof}

Now it suffices to bound ${\cal F}(\widetilde{D}-\widetilde{R})$. For an $n_1 \times n_2$ matrix $X$, consider the matrix norm 
$$\| X \|_{1,\infty} := \sum_{i=1}^{n_1} \max_{1 \leq j \leq n_2} | X_{i,j}|  =  \sum_{i=1}^{n_1} \| X_{i,.} \|_{\infty}.$$

The following lemma will be useful:

\begin{lemma}
\label{lem.bounds on delta}
${\cal F}(M)\leq \min\left\{ \|M \|_{1,\infty}, \hspace{.5mm} \min\{k, r\}\|M\|_{2 \rightarrow 2} \right\}$ where $r = \text{rank}(M)$.
\end{lemma}

\begin{proof}
The first bound follows from the classical version of H\"{o}lder's inequality (recalling that $X_{i, j} \geq 0$ and $X 1 = 1$ by design):
\[
|\operatorname{Tr}(MX)|
\leq \sum_{i=1}^N \sum_{j=1}^N | M_{i,j} X_{i,j} | \\
\leq \sum_{i=1}^N \| M_{i,.} \|_{\infty} \left( \sum_{j=1}^N | X_{i,j} | \right) \\
= \sum_{i=1}^N \| M_{i,.} \|_{\infty} 
\]

The second bound is a consequence of Von Neumann's trace inequality: if the singular values of $X$ and $M$ are respectively $\alpha_1\geq \ldots \geq \alpha_N $ and $\beta_1\geq \ldots \geq \beta_N $ then 

\[
|\operatorname{Tr}(MX)|
\leq\sum_{i=1}^N \alpha_i \beta_i
\]

Since $X$ is feasible in \eqref{eq.delta} we have $\alpha_1\leq 1$ and $\sum_{i=1}^N \alpha_{i}\leq k$. Using that $\operatorname{rank}(M)= r$  we get

\[
|\operatorname{Tr}(MX)|
\leq\sum_{i=1}^k \beta_i \leq \min\{k, r\}\|M\|_{2\rightarrow2}
\qedhere
\]

\end{proof}

\begin{proof}[Proof of Theorem~\ref{thm.distance typically small}]
Write $x_{t,i}=r_{t,i}+\gamma_t$.
Then
\begin{align*}
(D_{ab})_{ij}
&=\|x_{a,i}-x_{b,j}\|_2^2\\
&=\|(r_{a,i}+\gamma_a)-(r_{b,j}+\gamma_b)\|_2^2
=\|r_{a,i}-r_{b,j}\|_2^2+2\langle r_{a,i}-r_{b,j},\gamma_a-\gamma_b\rangle+\|\gamma_a-\gamma_b\|_2^2.
\end{align*}
Furthermore,
\[
\|r_{a,i}-r_{b,j}\|_2^2
=\|r_{a,i}\|_2^2-2\langle r_{a,i},r_{b,j}\rangle+\|r_{b,j}\|_2^2
=((\mu 1^\top + G^\top G + 1\mu^\top)_{ab})_{ij},
\]
where $G$ is the matrix whose $(a,i)$th column is $r_{a,i}$, and $\mu$ is the column vector whose $(a,i)$th entry is $\|r_{a,i}\|_2^2$.
Recall that 
$$(R_{ab})_{ij}= \xi + \Delta_{ab}^2/2 +  
\max\left\{ 0, \Delta_{ab}^2/2 +2 \langle r_{a,i}-r_{b,j}, \gamma_a-\gamma_b \rangle\right\}$$

Then $P_{1^\perp} (D-R) P_{1^\perp} = P_{1^\perp}G^\top G P_{1^\perp} + P_{1^\perp}FP_{1^\perp}$ where 
$$(F_{ab})_{ij}=\left\{ 
\begin{array}{ll} \Delta_{ab}^2/2 + 2\langle r_{a,i}-r_{b,j}, \gamma_a-\gamma_b \rangle
& \text{if }\quad 2\langle r_{a,i}-r_{b,j}, \gamma_a-\gamma_b \rangle \leq -\Delta_{ab}^2/2
\\
0 & \text{otherwise.}
\end{array}\right. $$

Considering Lemma~\ref{lem.bounds on delta} and that $\operatorname{rank}(G^\top G)\leq m$ we will bound 
\begin{equation}{\cal F}(M)\leq \min\{ k, m \}\|P_{1^\perp}G^\top G P_{1^\perp}\|_{2\to2} + \frac{1}{n_{\min}}\|P_{1^\perp}FP_{1^\perp}\|_{1,\infty}. \label{bound.delta}
\end{equation}

For the first term:
\[
\|P_{1^\perp}G^\top GP_{1^\perp}\|_{2\rightarrow2}
\leq\|G^\top G\|_{2\rightarrow2}
=\|G^\top \|_{2\rightarrow2}^2.
\]
Note that if the rows $X^{(t)}_i$, $i=1,\ldots n_t$ of $G^\top$ come from a distribution with second moment matrix $\Sigma_t$, then $X_i^{(t)}$ has the same distribution as $\Sigma_t^{1/2}g$, where $g$ is an isotropic random vector. Then $\|G^\top\|\leq \sigma_{\max} \|\tilde G^\top \|$ where the rows of $\tilde G^\top$ are isotropic random vectors. 

By Theorem~5.39 in~\cite{Vershynin:12}, we have that there exist $c_1$ and $c_2$ constants depending only on the subgaussian norm of the rows of $G$ such that with probability $\geq1-\eta$:
\[
\|G^\top \|_{2\rightarrow2}
\leq\sigma_{\max}\Big(\sqrt{N}+c_1\sqrt{m}+\sqrt{c_2\log(2/\eta)}\Big).
\]

Note that by Corollary~3.35, when the rows of $G^\top$ are Gaussian random vectors we have the result for $c_1=1$ and $c_2=2$.

For bounding the second term in \eqref{bound.delta}, the triangle inequality gives $ \|P_{1^\perp}F P_{1^\perp}\|_{1,\infty} \leq 4 \|F\|_{1,\infty}$.
In order to get a handle on $\|F\|_{1,\infty}$ we first compute the expected value of its entries using that $|2 \langle r_{a,i}-r_{b,j}, \gamma_a-\gamma_b \rangle|$ obeys a folded subgaussian distribution, coming from a subgaussian with variance at most $8\sigma_{\max}^2\Delta_{ab}^2$:
\begin{align*}
\mathbb E |(F_{ab})_{ij}| 
&\leq \left(\Delta_{ab}^2/2 + \mathbb E |2 \langle r_{a,i}-r_{b,j}, \gamma_a-\gamma_b \rangle| \right)\,\mathbb P\left(2\langle r_{a,i}-r_{b,j}, \gamma_a-\gamma_b \rangle < -\Delta_{ab}^2/2 \right) \\
&\leq \left(\frac{\Delta_{ab}^2}{2} + \frac{4\sigma_{\max}\Delta_{ab}}{\sqrt\pi} \right)\exp\left(-\frac{\Delta_{ab}^2}{64\sigma_{\max}^2}\right)
\\
&\leq \Delta_{ab}^2 \exp \left(-\frac{\Delta_{ab}^2}{64\sigma_{\max}^2}\right) \text{ assuming } \Delta_{\min}^2>16k\sigma^2_{\max},\quad k\geq 2
\\
&\leq \Delta_{ab}^2 \frac{64^2\sigma_{\max}^4}{\Delta_{ab}^4} \text{ using } e^{-x}\leq\frac{1}{x^2} 
\text{ for }x>0.
\\
&\leq -\frac{256\sigma_{\max}^2}{k} \text{ using again } \Delta_{\min}^2>16k\sigma^2_{\max},\quad k\geq 2
\\
&=O(\sigma_{\max}^2/k)
\end{align*}
Now we can write $F=2(L-L^\top)$ where $L_{a,i}:=(L_{ab})_{ij}\in \{\langle r_{a,i}, \gamma_a-\gamma_b \rangle, 0\}$ has independent rows, and $\mathbb E |(L_{ab})_{ij}| \leq \mathbb E |(F_{ab})_{ij}| =O(\sigma_{\max}^2/k)$. We can then bound 
$$
\|F\|_{1,\infty} \leq 4 \| L \|_{1, \infty} \leq \| L^{small} \|_{1,1}
$$ 
where $L^{small} \in \mathbb{R}^{N \times k }$ is a submatrix of distinct columns.

Then we have a high-probability estimate:
\begin{align*}
\mathbb P (\|L^{small} \|_{1,1}>t)
&\leq  \mathbb P \left( 2k\sum_{a=1}^k \sum_{i=1}^{n_a} |L_{a,i}|>t\right) 
\leq  \mathbb P \left( \sum_{a=1}^k \sum_{i=1}^{n_a} \left(|L_{a,i}| - \mathbb E |L_{a,i}|\right) >\frac{t}{2k}-c_3\sigma_{\max}^2n_{\max} \right) 
\end{align*}

 Using that $L_{a,i}$ are independent subgaussian random variables, we know there exist constants $c_4,c_5\geq 0$ such that 

$$\mathbb P \left( \sum_{a=1}^k \sum_{i=1}^{n_a} \left(|L_{a,i}| - \mathbb E |L_{a,i}|\right) >  u \right) \leq c_4\exp\left(-c_5\frac{u^2}{N}\right)$$

So, choosing $t=2 c_3 k n_{\max} \sigma_{\max}^2 + \sqrt{\frac{N}{c_5}\log{\frac{c_4}{\eta}}}$, we get that with probability at least $1-\eta$

$$\|P_{1^\perp}FP_{1^\perp}\|_{1,\infty} \leq 8 c_3 k n_{\max} \sigma_{\max}^2 + 4\sqrt{\frac{N}{c_5}\log{\frac{c_4}{\eta}}}$$

Putting everything together, we get that there exist constants $C_1,C_2,C_3$ such that with probability at least $1-2\eta$
\begin{align*}
\|X_D-X_R\|_\mathrm{F}^2
&\leq\frac{5}{n_{\min}\Delta_\mathrm{min}^2}\operatorname{Tr}(R(X_D-X_R))\\
&\leq\frac{10}{n_{\min}\Delta_\mathrm{min}^2}{\cal F}(\widetilde{D}-\widetilde{R})\\
&\leq C_1 \frac{ \min\{k, m\} \left( \sqrt N+ c_1 \sqrt m+\sqrt{ c_2\log(2/\eta)} \right)^2 \sigma_{\max}^2}{n_{\min}\Delta_\mathrm{min}^2} +C_2 \frac{k n_{\max}\sigma_{\max}^2}{n_{\min}\Delta_{\min}^2}+ C_3 \frac{\sqrt{ N \log{c_4/\eta }} }{n_{\min} \Delta_{\min}^2 }.
\end{align*}

If additionally we require $N > \max \{ c_1 m, c_2 \log(2/\eta), \log(c_4/\eta) \},$ we get 

$$\|X_D-X_R\|_\mathrm{F}^2 \leq C \frac{k \alpha\sigma_{\max}^2 (\alpha + \min\{k, m\})}{\Delta_{\min}^2 }$$
 
Rearranging gives the result.
\end{proof}

\section{Denoising} \label{sec:denoising}

In the special case where each Gaussian is spherical with the same entrywise variance $\sigma^2$ and the same number $n$ of samples, the main result is says:
\[
\|X_D-X_R\|_\mathrm{F}^2
\lesssim \frac{k^2 \sigma^2}{\Delta_\mathrm{min}^2}
\]
with high probability as $n\rightarrow\infty$.

Let $P$ denote the $m\times N$ matrix whose $(a,i)$th column is $x_{a,i}$.
Then $PX_R$ is an $m\times N$ matrix whose $(a,i)$th column is $\tilde \gamma_a $, a good estimate of $\gamma_a$, and so one might expect $PX_D$ to have its columns be close to the $\tilde\gamma_a$'s.
This is precisely what the following theorem gives:

\begin{theorem}
\label{thm.denoise}
Suppose $\sigma\lesssim \Delta_\mathrm{min}/\sqrt k$.
Let $P$ denote the $m\times N$ matrix whose $(a,i)$th column is $x_{a,i}$, and let $c_{a,i}$ denote the $(a,i)$th column of $PX_D$.
Then
\[
\frac{1}{N}\sum_{a=1}^k\sum_{i=1}^n\|c_{a,i}-\tilde\gamma_a\|_2^2
\lesssim \frac{\|\Gamma\|_{2\rightarrow2}^2}{\Delta_\mathrm{min}^2}\cdot k\sigma^2
\]
with high probability as $n\rightarrow\infty$.
Here, the $a$th column of $\Gamma$ is $\tilde \gamma_a-\frac{1}{k}\sum_{b=1}^k\tilde\gamma_b$.
\end{theorem}

The proof can be found at the end of this section.
For comparison,
\begin{equation}
\label{eq.comparison}
\mathbb{E}\bigg[\frac{1}{N}\sum_{a=1}^k\sum_{i=1}^n\|x_{a,i}-\gamma_a\|_2^2\bigg]
=m\sigma^2,
\end{equation}
meaning the $c_{a,i}$'s serve as ``denoised'' versions of the $x_{a,i}$'s provided $\|\Gamma\|_{2\rightarrow2}$ is not too large compared to $\Delta_\mathrm{min}$.
The following lemma investigates this provision:

\begin{lemma}
\label{lem.shape}
For every choice of $\{\tilde \gamma_a\}_{a=1}^k$, we have
\[
\frac{\|\Gamma\|_{2\rightarrow2}^2}{\Delta_\mathrm{min}^2}
\geq \frac{1}{2},
\]
with equality if $\{\tilde\gamma_a\}_{a=1}^k$ is a simplex.
More generally, if the following are satisfied simultaneously:
\begin{itemize}
\item[(i)] $\sum_{a=1}^k\tilde\gamma_a=0$,
\item[(ii)] $\|\tilde\gamma_a\|_2\asymp 1$ for every $a\in\{1,\ldots,k\}$, and
\item[(iii)] $|\langle \tilde\gamma_a,\tilde\gamma_b\rangle|\lesssim 1/k$ for every $a,b\in\{1,\ldots,k\}$ with $a\neq b$,
\end{itemize}
then
\[
\frac{\|\Gamma\|_{2\rightarrow2}^2}{\Delta_\mathrm{min}^2}
\lesssim 1.
\]
\end{lemma}

See the end of the section for the proof.
Plugging these estimates for $\|\Gamma\|_{2\rightarrow2}^2/\Delta_\mathrm{min}^2$ into Theorem~\ref{thm.denoise} shows that the $c_{a,i}$'s in this case exhibit denoising to an extent that the $m$ in \eqref{eq.comparison} can be replaced with $k$:
\[
\frac{1}{N}\sum_{a=1}^k\sum_{i=1}^n\|c_{a,i}-\tilde\gamma_a\|_2^2
\lesssim k\sigma^2.
\]
For more general choices of $\{\tilde\gamma_a\}_{a=1}^k$, one may attempt to estimate $\|\Gamma\|_{2\rightarrow2}$ in terms of $\Delta_\mathrm{max}$, but this comes with a bit of loss in the denoising estimate:

\begin{corollary}
If $k\sigma\lesssim\Delta_\mathrm{min}\leq\Delta_\mathrm{max}\lesssim K\sigma$, then
\[
\displaystyle{\frac{1}{N}\sum_{a=1}^k\sum_{i=1}^n\|c_{a,i}-\tilde\gamma_a\|_2^2\lesssim K^2\sigma^2}
\]
with high probability as $n\rightarrow\infty$.
\end{corollary}

Indeed, this doesn't guarantee denoising unless $k\lesssim K\leq\sqrt{m}$.
To prove this corollary, apply the following string of inequalities to Theorem~\ref{thm.denoise}:
\[
\|\Gamma\|_{2\rightarrow2}^2
\leq \|\Gamma\|_\mathrm{F}^2
\leq k\Delta_\mathrm{max}^2
\lesssim kK^2\sigma^2,
\]
where the second inequality uses the following lemma:

\begin{lemma}
If $\sum_{a=1}^k\tilde\gamma_a=0$, then $\|\tilde\gamma_a\|_2\leq\Delta_\mathrm{max}$ for every $a$.
\end{lemma}

\begin{proof}
Fix $a$.
Then
\[
\min_{b\neq a}\bigg\langle \tilde\gamma_b,\frac{\tilde\gamma_a}{\|\tilde\gamma_a\|_2}\bigg\rangle
\leq\frac{1}{k-1}\sum_{\substack{b=1\\b\neq a}}^k\bigg\langle\tilde\gamma_b,\frac{\tilde\gamma_a}{\|\tilde\gamma_a\|_2}\bigg\rangle
=\frac{1}{k-1}\bigg\langle \sum_{b=1}^k\tilde\gamma_b -\tilde\gamma_a,\frac{\tilde\gamma_a}{\|\tilde\gamma_a\|_2}\bigg\rangle
=-\frac{1}{k-1}\|\tilde\gamma_a\|_2.
\]
Let $b(a)$ denote the minimizer.
Then Cauchy--Schwarz gives
\[
\Delta_\mathrm{max}
\geq \|\tilde\gamma_a-\tilde\gamma_{b(a)}\|_2
\geq \bigg\langle \tilde\gamma_a-\tilde\gamma_{b(a)},\frac{\tilde\gamma_a}{\|\tilde\gamma_a\|_2}\bigg\rangle
\geq \|\tilde\gamma_a\|_2+\tfrac{1}{k-1}\|\tilde\gamma_a\|_2
\geq \|\tilde\gamma_a\|_2.\qedhere
\]
\end{proof}

\begin{proof}[Proof of Theorem~\ref{thm.denoise}]
Without loss of generality, we have $\sum_{a=1}^k\tilde\gamma_a=0$.
Write
\begin{equation}
\label{eq.denoise1}
\sum_{a=1}^k\sum_{i=1}^n\|c_{a,i}-\tilde\gamma_a\|_2^2
=\|P(X_D-X_R)\|_\mathrm{F}^2
\leq\|P\|_{2\rightarrow2}^2\|X_D-X_R\|_\mathrm{F}^2.
\end{equation}
Decompose $P=\Gamma\otimes1^\top+G$, where $1$ is $n$-dimensional and $G$ has i.i.d.\ entries from $\mathcal{N}(0,\sigma^2)$.
Observe that
\begin{equation}
\label{eq.denoise2}
\|\Gamma\otimes1^\top\|_{2\rightarrow2}^2
=\|(\Gamma\otimes1^\top)(\Gamma\otimes1^\top)^\top\|_{2\rightarrow2}
=\|n\Gamma\Gamma^\top\|_{2\rightarrow2}
=n\|\Gamma\|_{2\rightarrow2}^2.
\end{equation}
Also, Corollary~5.35 in~\cite{Vershynin:12} gives that
\begin{equation}
\label{eq.denoise3}
\|G\|_{2\rightarrow2}\lesssim(\sqrt{N}+\sqrt{m})\sigma\lesssim\sqrt{N}\sigma
\end{equation}
with probability $\geq1-e^{-\Omega_m(N)}$.
The result then follows from estimating $\|P\|_{2\rightarrow2}$ with \eqref{eq.denoise2} and \eqref{eq.denoise3} by triangle inequality, plugging into \eqref{eq.denoise1}, and then applying Theorem~\ref{thm.distance typically small}.
\end{proof}

\begin{proof}[Proof of Lemma~\ref{lem.shape}]
Since $\|\Gamma x\|_2\leq\|\Gamma\|_{2\rightarrow2}\|x\|_2$ for every $x$, we have that
\[
\|\Gamma\|_{2\rightarrow2}^2
\geq\frac{\|\tilde\gamma_a-\tilde\gamma_b\|_2^2}{2}
\]
for every $a$ and $b$, and so
\[
\frac{\|\Gamma\|_{2\rightarrow2}^2}{\Delta_\mathrm{min}^2}
\geq \frac{1}{2}\cdot\frac{\Delta_\mathrm{max}^2}{\Delta_\mathrm{min}^2}
\geq \frac{1}{2}.
\]
For the second part, let $\{\tilde\gamma_a\}_{a=1}^k$ be a simplex.
Without loss of generality, $\{\tilde\gamma_a\}_{a=1}^k$ is centered at the origin, each point having unit 2-norm.
Then $\langle \tilde\gamma_1,\tilde\gamma_2\rangle=-1/(k-1)$, and so
\[
\Delta_\mathrm{min}^2
=\|\tilde\gamma_1-\tilde\gamma_2\|_2^2
=\|\tilde\gamma_1\|_2^2+\|\tilde\gamma_2\|_2^2-2\langle \tilde\gamma_1,\tilde\gamma_2\rangle
=\frac{2k}{k-1}.
\]
Next, we write
\[
\Gamma^\top\Gamma
=\frac{k}{k-1}I-\frac{1}{k-1}11^\top,
\]
and conclude that $\|\Gamma\|_{2\rightarrow2}^2=\|\Gamma^\top\Gamma\|_{2\rightarrow2}=k/(k-1)$.
Combining with our expression for $\Delta_\mathrm{min}^2$ then gives the result.
For the last part, pick $a$ and $b$ such that $\Delta_\mathrm{min}=\|\tilde\gamma_a-\tilde\gamma_b\|_2$.
Then
\[
\Delta_\mathrm{min}^2
=\|\tilde\gamma_a\|_2^2+\|\tilde\gamma_b\|_2^2-2\langle \tilde\gamma_a,\tilde\gamma_b\rangle
\gtrsim 2-2/k.
\]
Also, Gershgorin implies
\[
\|\Gamma\|_{2\rightarrow2}^2
=\|\Gamma^\top\Gamma\|_{2\rightarrow2}
\lesssim 1+(k-1)/k,
\]
and so combining these estimates gives the result.
\end{proof}

\section{Rounding} \label{sec:rounding}

\begin{theorem}
\label{thm.rounding}
Take $\epsilon<\Delta_\mathrm{min}/8$, suppose
\[
\#\Big\{(a,i):\|c_{a,i}-\tilde\gamma_a\|_2>\epsilon\Big\}<\frac{n}{2},
\]
and consider the graph $G$ of vertices $\{c_{a,i}\}_{i=1,}^n\ \!_{a=1}^k$ such that $c_{a,i}\leftrightarrow c_{b,j}$ if $\|c_{a,i}-c_{b,j}\|_2\leq 2\epsilon$.
For each $i=1,\ldots,k$, select the vertex $v_i$ of maximum degree (breaking ties arbitrarily) and update $G$ by removing every vertex $w$ such that $\|w-v_i\|_2\leq 4\epsilon$.
Then there exists a permutation $\pi$ on $\{1,\ldots,k\}$ such that
\[
\|v_i-\tilde\gamma_{\pi(i)}\|_2\leq3\epsilon
\]
for every $i\in\{1,\ldots,k\}$.
\end{theorem}

\begin{proof}
Let $B(x,r)$ denote the closed $2$-ball of radius $r$ centered at $x$.
For each $i$, we will determine $\pi(i)$ at the conclusion of iteration $i$.
Denote $R_1:=\{1,\ldots,k\}$ and $R_{i+1}:=R_i\setminus\{\pi(i)\}$ for each $i=2,\ldots,k-1$.
We claim that the following hold at the beginning of each iteration $i$:
\begin{itemize}
\item[(i)]
$<n/2$ vertices lie outside $\displaystyle{\bigcup_{a\in R_i} B(\tilde\gamma_a,\epsilon)}$,
\item[(ii)]
$\geq n/2$ vertices lie inside $B(v_i,2\epsilon)$, and
\item[(iii)]
there exists a unique $a\in R_i$ such that $\|v_i-\tilde\gamma_a\|_2\leq 3\epsilon$.
\end{itemize}

First, we show that for each $i$, (i) and (ii) together imply (iii).
Indeed, there are enough vertices in $B(v_i,2\epsilon)$ that one of them must reside in $B(\tilde\gamma_a,\epsilon)$ for some $a\in R_i$.
Furthermore, this $a$ is unique since $\epsilon<\Delta_\mathrm{min}/6$.
By triangle inequality, we have $\|v_i-\tilde\gamma_a\|_2\leq3\epsilon$, and so we put $\pi(i):=a$.

We now prove (i) and (ii) by induction.
When $i=1$, we have (i) by assumption.
For (ii), note that each $B(\tilde\gamma_a,\epsilon)$ contains $\geq n/2$ of the vertices, and by triangle inequality, each has degree $\geq n/2-1$ in $G$.
As such, the vertex $v_1$ of maximum degree will have degree $\geq n/2-1$, thereby implying (ii).

Now suppose (i), (ii) and (iii) all hold for iteration $i<k$.
By triangle inequality, (iii) implies $B(\tilde\gamma_{\pi(i)},\epsilon)\subseteq B(v_i,4\epsilon)$.
As such, the $i$th iteration removes all vertices in $B(\tilde\gamma_{\pi(i)},\epsilon)$ so that (i) continues to hold for iteration $i+1$.
Next, $\epsilon<\Delta_\mathrm{min}/8$ and (iii) together imply that the removal of vertices in $B(v_i,4\epsilon)$ preserves the vertices in $B(\tilde\gamma_a,\epsilon)$ for every $a\in R_{i+1}$, and their degrees are still $\geq n/2-1$ by the same triangle argument as before.
Thus, (ii) holds for iteration $i+1$.
\end{proof}

\begin{corollary}
\label{cor.rounding the SDP}
Suppose $k\lesssim m$, and denote $S:=\|\Gamma\|_{2\rightarrow2}/\Delta_\mathrm{min}$.
Pick $\epsilon\asymp Sk\sigma$.
Perform the rounding scheme of Theorem~\ref{thm.rounding} to columns of $PX_D$.
Then with high probability, $\{v_i\}_{i=1}^k$ satisfies
\[
\|v_i-\tilde\gamma_{\pi(i)}\|_2
\lesssim Sk\sigma
\]
for some permutation $\pi$, provided $\sigma\lesssim \Delta_\mathrm{min}/(Sk)$.
\end{corollary}

By Lemma~\ref{lem.shape}, we have $S\lesssim 1$ in the best-case scenario.
In this case, our rounding scheme works in the regime $\sigma\lesssim\Delta_\mathrm{min}/k$. (Note that denoising is guaranteed in the regime $\sigma\lesssim\Delta_\mathrm{min}/\sqrt k$).
In general, the cost of rounding is a factor of $k$ in the average squared deviation of our estimates:
\[
\frac{1}{N}\sum_{a=1}^k\sum_{i=1}^n\|c_{a,i}-\tilde\gamma_a\|_2^2
\lesssim S^2k\sigma^2,
\qquad
\text{whereas}
\qquad
\frac{1}{k}\sum_{i=1}^k\|v_i-\tilde\gamma_{\pi(i)}\|_2^2
\lesssim S^2k^2\sigma^2.
\]
On the other hand, we are not told which of the points in $\{c_{a,i}\}_{i=1,}^n\ \!_{a=1}^k$ correspond to any given $\tilde\gamma_a$, whereas in rounding, we know that each $v_i$ corresponds to a distinct $\tilde\gamma_a$.

\begin{proof}[Proof of Corollary~\ref{cor.rounding the SDP}]
Draw $(a,i)$ uniformly from $\{1,\ldots,k\}\times\{1,\ldots,n\}$ and take $X$ to be the random variable $\|c_{a,i}-\tilde\gamma_a\|_2^2$.
Then Markov's inequality and Theorem~\ref{thm.denoise} together give
\begin{align*}
\#\Big\{(a,i):\|c_{a,i}-\tilde\gamma_a\|_2>\epsilon\Big\}
&=N\cdot\mathbb{P}(X>\epsilon^2)\\
&\leq\frac{N}{\epsilon^2}\cdot\frac{1}{N}\sum_{a=1}^k\sum_{i=1}^n\|c_{a,i}-\tilde\gamma_a\|_2^2
\lesssim \frac{N}{\epsilon^2}\cdot S^2k\sigma^2
\lesssim \frac{n}{2}.
\end{align*}
For Theorem~\ref{thm.rounding} to apply, it suffices to ensure $\epsilon<\Delta_\mathrm{min}/8$, which follows from $\sigma\lesssim\Delta_\mathrm{min}/(Sk)$.
\end{proof}

\section{Proof of Theorem~\ref{thm.significance of stable isogons}}
\label{sec.proof of thm.significance of stable isogons}

\begin{lemma}
\label{lem.decomposition of G}
Let $G$ be the symmetry group of a stable isogon $\Gamma\subseteq\mathbb{R}^m$, and let $K$ and $H$ denote the subgroups of $G$ that fix $W=\operatorname{span}(\Gamma)$ and its orthogonal complement $W^\perp$, respectively.
Then
\begin{itemize}
\item[(i)]
$G$ is the direct sum of $H$ and $K$,
\item[(ii)]
$H$ is finite,
\item[(iii)]
$H$ acts transitively on $\Gamma$, and
\item[(iv)]
$K$ is isomorphic to the orthogonal group $O(m-r)$, where $r$ is the dimension of $W$.
\end{itemize}
\end{lemma}

\begin{proof}
Pick $Q\in G$.
Then $Q$ permutes the points in $\Gamma$, and the permutation completely determines how $Q$ acts on $W$ by linearity.
In particular, $W$ is invariant under the action of $G$, which in turn implies the same for $W^\perp$.
Let $V$ denote an $m\times r$ matrix whose columns form an orthonormal basis for $W$, and let $V_\perp$ denote an $m\times (m-r)$ matrix whose columns form an orthonormal basis for $W^\perp$.
Then $Q$ can be expressed as
\[
Q=\bigg[\begin{array}{cc}V&V_\perp\end{array}\bigg]\bigg[\begin{array}{cc} A&0\\0&B\end{array}\bigg] \bigg[\begin{array}{c}~~V^\top~~\\V_\perp^\top\end{array}\bigg]
\]
We see that $Q\in H$ when $B=I$ and $Q\in K$ when $A=I$.
Let $Q_H$ denote the ``projection'' of $Q$ onto $H$, obtained by replacing $B$ with $I$, and similarly for $Q_K$.

For (i), it suffices to show that $H$ and $K$ are normal subgroups of $G$, that $H\cap K=\{I\}$, and that $G$ is generated by $H$ and $K$.
The first is obtained by observing that $K$ is the kernel of the homomorphism $Q\mapsto Q_H$, and similarly, $H$ is the kernel of $Q\mapsto Q_K$.
The second follows from the observation that $Q\in H\cap K$ implies $A=I$ and $B=I$.
The last follows from the observation that every $Q\in G$ can be factored as $Q_HQ_K$.

For (ii), we note that $A$ is completely determined by how $Q$ permutes the points in $\Gamma$, of which $\leq k!$ possibilities are available.

For (iii), we know that by (si2), for every pair $\gamma,\gamma'\in \Gamma$, there exists $Q\in G$ such that $G\gamma=\gamma'$.
Consider the factorization $Q=Q_HQ_K$.
Since $Q_K\gamma=\gamma$, we therefore have $Q_H\gamma=\gamma'$, meaning $H$ also acts transitively on $\Gamma$.

For (iv), we first note that $B\in O(m-r)$ is necessary in order to have $Q\in O(m)$.
Now pick any $B\in O(m-r)$.
Then
\begin{equation}
\label{eq.form of members of K}
Q=\bigg[\begin{array}{cc}V&V_\perp\end{array}\bigg]\bigg[\begin{array}{cc} I&0\\0&B\end{array}\bigg] \bigg[\begin{array}{c}~~V^\top~~\\V_\perp^\top\end{array}\bigg]
\end{equation}
has the effect of fixing each point in $\Gamma$, meaning $Q\in G$ (and therefore $Q\in K$).
As such, $K$ is the set of all $Q$ of the form~\eqref{eq.form of members of K}.
\end{proof}

\begin{lemma}
\label{lem.stable isogons are centered}
For any stable isogon $\Gamma$, we have $\sum_{\gamma\in\Gamma}\gamma=0$.
\end{lemma}

\begin{proof}
By (si1), we have $|\Gamma|>1$.
In the special case where $|\Gamma|=2$ and the points in $\Gamma$ are linearly dependent, write $\Gamma=\{\gamma_1,\gamma_2\}$ with $\gamma_1\neq \gamma_2$.
By (si2), we know there exists $Q\in G$ such that $Q\gamma_1=\gamma_2$, and so $\|\gamma_1\|_2=\|\gamma_2\|_2$.
This combined with the assumed linear dependence gives $\gamma_1=\pm\gamma_2$, and since $\gamma_1\neq\gamma_2$, we conclude $\gamma_1+\gamma_2=0$, as desired.

In the remaining case, $\Gamma$ contains two points (say, $\gamma_1$ and $\gamma_2$) that are linearly independent.
Fix $\gamma_0\in\Gamma$.
By Lemma~\ref{lem.decomposition of G}(iii), the orbit $H\gamma_0$ is all of $\Gamma$.
Consider the map from $H$ onto $\Gamma$ given by $f\colon U\mapsto U\gamma_0$ for all $U\in H$.
The preimage of any member of $\Gamma$ is a left coset of the stabilizer $H_{\gamma_0}$ (which is finite by Lemma~\ref{lem.decomposition of G}(ii)), and so
\begin{equation}
\label{eq.pass to orbit}
\sum_{U\in H}U\gamma_0
=\sum_{\gamma\in\Gamma}\sum_{U\in f^{-1}(\gamma)}U\gamma_0
=|H_{\gamma_0}|\sum_{\gamma\in\Gamma}\gamma.
\end{equation}
Now pick any $Q\in G$, and consider the factorization $Q=Q_H Q_K$ with $Q_H\in H$ and $Q_K\in K$ (this exists uniquely by Lemma~\ref{lem.decomposition of G}(i)).
Since $U\gamma_0\in\operatorname{span}(\Gamma)$ for every $U\in H$, we have that $Q_KU\gamma_0=U\gamma_0$, and so
\begin{equation}
\label{eq.sum over orbit is fixed}
Q\sum_{U\in H}U\gamma_0
=\sum_{U\in H}Q_H U\gamma_0
=\sum_{U\in H}U\gamma_0,
\end{equation}
where the last step follows from the fact that multiplication by $Q_H\in H$ permutes the members of $H$.
Put $x=\sum_{U\in H}U\gamma_0$.
Then \eqref{eq.sum over orbit is fixed} gives that $Qx=x$ for every $Q\in G$, and therefore $Qx=x$ for every $Q\in G_{\gamma_1}\cup G_{\gamma_2}$.
By (si3), this then implies that $x\in\operatorname{span}\{\gamma_1\}\cap\operatorname{span}\{\gamma_2\}$, i.e., $x=0$.
Combining with \eqref{eq.pass to orbit} then gives the result.
\end{proof}

\begin{lemma}
\label{lem.pdf symmetry}
Let $\mathcal{D}$ be a mixture of equally weighted spherical Gaussians of equal variance centered at the points of a stable isogon $\Gamma=\{\gamma_t\}_{t=1}^k$.
If $X\sim\mathcal{D}$ and $Q$ is any member of the symmetry group $G$ of $\Gamma$, then $QX\sim\mathcal{D}$.
\end{lemma}

\begin{proof}
The probability density function of $\mathcal{D}$ is given by
\[
f(x)
=\frac{1}{k}\sum_{t=1}^k\frac{1}{(\sqrt{2\pi}\sigma)^m}e^{-\|x-\gamma_t\|_2^2/2\sigma^2}.
\]
As such, $f(Qx)=f(x)$ since $\|Qx-\gamma_t\|_2=\|x-Q^{-1}\gamma_t\|_2$ and $Q^{-1}\in G$ permutes the $\gamma_t$'s.
\end{proof}

\begin{proof}[Proof of Theorem~\ref{thm.significance of stable isogons}]
Pick $Q\in G_{\gamma_t}$.
Then $x\in V_t^{(\Gamma)}$ precisely when $x\in QV_t^{(\Gamma)}$.
As such, Lemma~\ref{lem.pdf symmetry} gives
\begin{align*}
\mu_t^{(\Gamma,\mathcal{D})}
=\mathop{\mathbb{E}}_{X\sim\mathcal{D}}\big[X\big|X\in V_t^{(\Gamma)}\big]
&=\mathop{\mathbb{E}}_{X\sim\mathcal{D}}\big[X\big|X\in QV_t^{(\Gamma)}\big]\\
&=\mathop{\mathbb{E}}_{X\sim\mathcal{D}}\big[X\big|Q^{-1}X\in V_t^{(\Gamma)}\big]\\
&=\mathop{\mathbb{E}}_{Y\sim\mathcal{D}}\big[QY\big|Y\in V_t^{(\Gamma)}\big]
=Q\mathop{\mathbb{E}}_{Y\sim\mathcal{D}}\big[Y\big|Y\in V_t^{(\Gamma)}\big]
=Q\mu_t^{(\Gamma,\mathcal{D})}.
\end{align*}
By (si3), this then implies that $\mu_t^{(\Gamma,\mathcal{D})}\in\operatorname{span}\{\gamma_t\}$.

At this point, we have $\mu_t^{(\Gamma,\mathcal{D})}=\alpha_t\gamma_t$, where $\alpha_t=\langle \mu_t^{(\Gamma,\mathcal{D})},\gamma_t\rangle/\|\gamma_t\|_2^2$.
For any given $s\in\{1,\ldots,k\}$, pick $Q$ such that $Q\gamma_t=\gamma_s$ (which exists by (si2)).
Then Lemma~\ref{lem.pdf symmetry} again gives
\begin{align*}
\langle \mu_s^{(\Gamma,\mathcal{D})},\gamma_s\rangle
&=\Big\langle \mathop{\mathbb{E}}_{X\sim\mathcal{D}}\big[X\big|X\in V_s^{(\Gamma)}\big],\gamma_s\Big\rangle\\
&=\mathop{\mathbb{E}}_{X\sim\mathcal{D}}\big[\langle X,\gamma_s\rangle\big|X\in V_s^{(\Gamma)}\big]\\
&=\mathop{\mathbb{E}}_{X\sim\mathcal{D}}\big[\langle X,\gamma_s\rangle\big|X\in QV_t^{(\Gamma)}\big]\\
&=\mathop{\mathbb{E}}_{Y\sim\mathcal{D}}\big[\langle QY,\gamma_s\rangle\big|Y\in V_t^{(\Gamma)}\big]
=\mathop{\mathbb{E}}_{Y\sim\mathcal{D}}\big[\langle Y,\gamma_t\rangle\big|Y\in V_t^{(\Gamma)}\big]
=\langle \mu_t^{(\Gamma,\mathcal{D})},\gamma_t\rangle,
\end{align*}
meaning $\alpha_t=\alpha$ for all $t$.
It remains to show that $\langle \mu_t^{(\Gamma,\mathcal{D})},\gamma_t\rangle>0$.
To this end, note that $\|x-\gamma_t\|_2^2<\|x-\gamma_s\|_2^2$ precisely when $\langle x,\gamma_t\rangle >\langle x,\gamma_s\rangle$, and so $x\in V_t^{(\Gamma)}$ if and only if $t=\arg\max_{a}\langle x,\gamma_a\rangle$.
By Lemma~\ref{lem.stable isogons are centered},
\[
\max_{a\in\{1,\ldots,k\}}\langle x,\gamma_a\rangle
\geq \frac{1}{k}\sum_{a=1}^k\langle x,\gamma_a\rangle
=\frac{1}{k}\bigg\langle x,\sum_{a=1}^k\gamma_a\bigg\rangle
=0,
\]
with equality only if the maximizer is not unique, and so we conclude that $x\in V_t^{(\Gamma)}$ only if $\langle x,\gamma_t\rangle>0$.
As such,
\[
\langle \mu_t^{(\Gamma,\mathcal{D})},\gamma_t\rangle
=\Big\langle \mathop{\mathbb{E}}_{X\sim\mathcal{D}}\big[X\big|X\in V_t^{(\Gamma)}\big],\gamma_t\Big\rangle
=\mathop{\mathbb{E}}_{X\sim\mathcal{D}}\big[\langle X,\gamma_t\rangle\big|X\in V_t^{(\Gamma)}\big]
>0,
\]
as desired.
\end{proof}

\section{Proof of Theorem~\ref{thm.main lower bound}}
\label{sec.proof of thm.main lower bound}

We start with the following:

\begin{lemma}
\label{lem.monotonicity}
Let $\Gamma\subseteq\mathbb{R}^d$ denote the standard orthoplex of dimension $d$, and for any given $c\geq0$, consider the mixture $\mathcal{D}_c$ of equally weighted spherical Gaussians of unit entrywise variance centered at the members of $c\Gamma$.
Let $V_1^{(\Gamma)}$ denote the Voronoi region corresponding to the first identity basis element, and define $\alpha_d\colon\mathbb{R}_{\geq0}\rightarrow\mathbb{R}$ by
\begin{equation}
\label{eq.defn of alpha}
\alpha_d(c)
:=\mathop{\mathbb{E}}_{X\sim\mathcal{D}_c}\big[X_1\big|X\in V_1^{(\Gamma)}\big]
=2d\int_{x_1=0}^\infty\int_{x_2=-x_1}^{x_1}\cdots\int_{x_d=-x_1}^{x_1} x_1 f(x;c)~dx_d\cdots dx_1,
\end{equation}
where $f(\cdot;c)$ denotes the probability density function of $\mathcal{D}_c$:
\[
f(x;c):=\frac{1}{2d}\sum_{t=1}^{2d} \frac{1}{(2\pi)^{d/2}}e^{-\|x-c\gamma_t\|_2^2/2}.
\]
Then $\alpha_d(c)\geq\alpha_d(0)$ for all $c\geq0$.
\end{lemma}

See the end of this section for the proof.
For context, our proof of Theorem~\ref{thm.main lower bound} requires a bound on $\alpha_d(c)$ for general $c\geq0$, and so Lemma~\ref{lem.monotonicity} allows us to pass to the easier-to-estimate quantity $\alpha_d(0)$.
The following lemmas estimate $\alpha_d(0)$:

\begin{lemma}
\label{lem.gaussian bound}
If $g\sim\mathcal{N}(0,I)$ in $\mathbb{R}^d$, then $\mathbb{E}\|g\|_\infty\gtrsim\sqrt{\log d}$.
\end{lemma}

\begin{proof}
When $d=1$, $\|g\|_\infty$ has half-normal distribution, and so its expected value is $\sqrt{2/\pi}$.
Otherwise, $d\geq2$.
Since $\|g\|_\infty\geq\max_ig(i)$, we will estimate $\mathbb{E}\max_ig(i)$.
To this end, take $z$ such that $\mathbb{P}(g(1)\geq z)=1/d$, denote $j:=\arg\max_ig(i)$, and condition on the event that $g(j)<z$, which occurs with probability $(1-1/d)^d$:
\begin{align}
\mathbb{E}g(j)
\nonumber
&=\mathbb{E}\Big[g(j)\Big|g(j)<z\Big]\cdot(1-1/d)^d+\mathbb{E}\Big[g(j)\Big|g(j)\geq z\Big]\cdot\Big(1-(1-1/d)^d\Big)\\
\label{eq.gaussian bound 1}
&\geq\frac{1}{2}\mathbb{E}\Big[g(j)\Big|g(j)<0\Big]\cdot(1/4)+z\cdot \Big(1-(1-1/d)^d\Big).
\end{align}
Since $g(j)\geq\frac{1}{d}\sum_{i=1}^d g(i)$, we have
\[
\mathbb{E}\Big[g(j)\Big|g(j)<0\Big]
\geq\mathbb{E}\bigg[\frac{1}{d}\sum_{i=1}^d g(i)\bigg|g(i)<0~\forall i\bigg]
=-\sqrt{\frac{2}{\pi}}.
\]
With this, we may continue \eqref{eq.gaussian bound 1} to get
\[
\mathbb{E}g(j)
\geq -\frac{1}{8}\sqrt{\frac{2}{\pi}}+z\cdot(1-e^{-1})
\gtrsim z
\gtrsim\sqrt{\log d},
\]
where the last step follows from rearranging $1/d=\mathbb{P}(g(1)\geq z)\geq e^{-O(z^2)}$.
\end{proof}

\begin{lemma}
\label{lem.bound on alpha zero}
The function $\alpha_d$ defined by \eqref{eq.defn of alpha} satisfies $\alpha_d(0)\gtrsim\sqrt{\log d}$.
\end{lemma}

\begin{proof}
It's straightforward to verify that $x\in V_t^{(\Gamma)}$ precisely when
\[
j=\arg\max_{i\in\{1,\ldots,d\}}|x(i)|
\qquad
\text{and}
\qquad
\gamma_t
=\operatorname{sign}(x(j))\cdot e_j,
\]
and so $x\in V_t^{(\Gamma)}$ implies $\langle x,\gamma_t\rangle=\|x\|_\infty$.
Letting $g\sim\mathcal{N}(0,I)$ in $\mathbb{R}^d$, then
\[
\alpha_d(0)
=\mathbb{E}\Big[\langle g,\gamma_1\rangle\Big|g\in V_1^{(\Gamma)}\Big]
=\mathbb{E}\Big[\|g\|_\infty\Big|g\in V_1^{(\Gamma)}\Big]
=\mathbb{E}\|g\|_\infty,
\]
where the last step follows from the fact that $\|\Pi g\|_\infty$ has the same distribution for every signed permutation $\Pi$, since this implies that the random variable is independent of the event $g\in V_1^{(\Gamma)}$.
The result then follows from Lemma~\ref{lem.gaussian bound}.
\end{proof}

We are now ready to prove the theorem of interest:

\begin{proof}[Proof of Theorem~\ref{thm.main lower bound}]
Take $\Gamma$ to be the standard orthoplex of dimension $d=k/2$.
Observe that Theorem~\ref{thm.significance of stable isogons} along with a change of variables in \eqref{eq.defn of alpha} gives
\[
\mu_t^{(\Gamma,\mathcal{D})}
=\alpha \gamma_t
=\sigma \alpha_d(1/\sigma)\gamma_t.
\]
This then implies
\begin{align}
\|\mu_t^{(\Gamma,\mathcal{D})}-\gamma_t\|_2
=|\langle \mu_t^{(\Gamma,\mathcal{D})}-\gamma_t,\gamma_t\rangle|
\nonumber
&\geq|\langle \mu_t^{(\Gamma,\mathcal{D})},\gamma_t\rangle|-|\langle \gamma_t,\gamma_t\rangle|\\
\label{eq.RHS}
&=\sigma\alpha_d(1/\sigma)-\Delta_\mathrm{min}/\sqrt{2}
\geq\sigma\alpha_d(0)-\Delta_\mathrm{min}/\sqrt{2},
\end{align}
where the last step follows from Lemma~\ref{lem.monotonicity}.
At this point, we consider two cases.
In the first case, $\eqref{eq.RHS}\geq\sigma\alpha_d(0)/2$, which implies
\[
\|\mu_t^{(\Gamma,\mathcal{D})}-\gamma_t\|_2
\geq\sigma\alpha_d(0)-\Delta_\mathrm{min}/\sqrt{2}
\geq\sigma\alpha_d(0)/2
\gtrsim\sigma\sqrt{\log k},
\]
where the last step follows from Lemma~\ref{lem.bound on alpha zero}.
Since this bound is independent of $t$, we then get
\[
\min_{t\in\{1,\ldots,k\}}\|\mu_t^{(\Gamma,\mathcal{D})}-\gamma_t\|_2
\gtrsim \sigma\sqrt{\log k}.
\]
In the remaining case, we have $\eqref{eq.RHS}<\sigma\alpha_d(0)/2$ which one may rearrange to get
\[
\sigma
<\frac{2}{\alpha_d(0)}\cdot\frac{\Delta_\mathrm{min}}{\sqrt{2}}
\lesssim \Delta_\mathrm{min}/\sqrt{\log k}.
\qedhere
\]
\end{proof}

\begin{proof}[Proof of Lemma~\ref{lem.monotonicity}]
We will show that $\alpha_d(c)$ has a nonnegative derivative, and the result will follow from the mean value theorem.
First, we write $\alpha_d(c)=\sum_{t=1}^{2d} I_t(c)$, where
\begin{equation}
\label{eq.defn I_t}
I_t(c)
:=\frac{1}{(2\pi)^{d/2}}\int_{x\in V_1^{(\Gamma)}}x_1e^{-\|x-c\gamma_t\|_2^2/2}dx.
\end{equation}
Denote $\gamma_1=e_1$ and $\gamma_2=-e_1$, i.e., the first identity basis element and its negative.
We claim that $I_s(\cdot)=I_t(\cdot)$ whenever $s,t\in\{3,\ldots,2d\}$.
This can be seen by changing variables in \eqref{eq.defn I_t} with any signed permutation that fixes $\gamma_1$ (and therefore $\gamma_2$ and $V_1^{(\Gamma)}$), and that sends $\gamma_s$ to $\gamma_t$.
As such, we have $\alpha_d(c)=I_1(c)+I_2(c)+(2d-2)I_3(c)$, where we take $\gamma_3=e_2$ without loss of generality.
At this point, we factor out the $x_1$-dependence in the integrands of $I_1(c)$ and $I_2(c)$ and observe that
\[
\int_{x_2=-x_1}^{x_1}\cdots\int_{x_d=-x_1}^{x_1} e^{-(x_2^2+\cdots+x_d^2)/2}dx_d\cdots dx_2
=\bigg(\int_{-x_1}^{x_1}e^{-z^2/2}dz\bigg)^{d-1}
=(2\pi)^{(d-1)/2}\operatorname{erf}^{d-1}(x_1/\sqrt{2})
\]
to get
\[
I_1(c)+I_2(c)
=\frac{1}{\sqrt{2\pi}}\int_0^\infty x\Big(e^{-(x-c)^2/2}+e^{-(x+c)^2/2}\Big)\operatorname{erf}^{d-1}(x/\sqrt{2})dx.
\]
Similarly,
\[
I_3(c)
=\frac{1}{2\pi}\int_{x_1=0}^\infty x_1e^{-x_1^2/2}\bigg(\int_{x_2=-x_1}^{x_1}e^{-(x_2-c)^2/2}dx_2\bigg)\operatorname{erf}^{d-2}(x_1/\sqrt{2})dx_1.
\]
At this point, we apply differentiation under the integral sign to get
\begin{align*}
I_1'(c)+I_2'(c)
&=\frac{1}{\sqrt{2\pi}}\int_0^\infty x\Big((x-c)e^{-(x-c)^2/2}-(x+c)e^{-(x+c)^2/2}\Big)\operatorname{erf}^{d-1}(x/\sqrt{2})dx,\\
I_3'(c)
&=\frac{1}{2\pi}\int_0^\infty xe^{-x^2/2}\Big(e^{-(x+c)^2/2}-e^{-(x-c)^2/2}\Big)\operatorname{erf}^{d-2}(x/\sqrt{2})dx.
\end{align*}
To continue, note that
\[
\frac{d}{dx}\operatorname{erf}^{d-1}(x/\sqrt{2})
=\frac{2(d-1)}{\sqrt{2\pi}}e^{-x^2/2}\operatorname{erf}^{d-2}(x/\sqrt{2}).
\]
With this, we integrate by parts to change the expression for $I_3'(c)$:
\[
I_3'(c)
=\frac{1}{2d-2}\bigg(-I_1'(c)-I_2'(c)+\frac{1}{\sqrt{2\pi}}\int_0^\infty\Big(e^{-(x-c)^2/2}-e^{-(x+c)^2/2}\Big)\operatorname{erf}^{d-1}(x/\sqrt{2})dx\bigg).
\]
Overall, we have
\[
\alpha_d'(c)
=I_1'(c)+I_2'(c)+(2d-2)I_3'(c)
=\frac{1}{\sqrt{2\pi}}\int_0^\infty\Big(e^{-(x-c)^2/2}-e^{-(x+c)^2/2}\Big)\operatorname{erf}^{d-1}(x/\sqrt{2})dx,
\]
which is nonnegative since the integrand is everywhere nonnegative.
\end{proof}

\section*{Acknowledgments}

The authors thank Pablo Parrilo for suggesting SDPNAL+ as a solver for the $k$-means semidefinite program, and David Bowie for the music that sustained our investigation of stable isogons and the Voronoi Means Conjecture.
DGM was supported by an AFOSR Young Investigator Research Program
award, NSF Grant No. DMS-1321779, and AFOSR Grant No. F4FGA05076J002.
RW was supported in part by an NSF CAREER grant and ASOFR Young Investigator Award 9550-13-1-0125. The views expressed in this
article are those of the authors and do not reflect the official
policy or position of the United States Air Force, Department of
Defense, or the U.S. Government.

\bibliographystyle{alpha}
\bibliography{literature_review}

\end{document}